\documentclass{article}
\PassOptionsToPackage{numbers,compress,sort}{natbib}
\usepackage[final]{neurips_2022}
\usepackage{soul}
\sethlcolor{white}
\newcommand{\rebuttal}[1]{\hl{#1}}

\usepackage[utf8]{inputenc} %
\usepackage[T1]{fontenc}    %
\usepackage[colorlinks=true, linkcolor=red, urlcolor=blue, citecolor=blue, anchorcolor=blue,backref=page]{hyperref}       %
\usepackage{amsmath}
\usepackage{amsthm}
\usepackage{amssymb}
\usepackage{url}            %
\usepackage{booktabs}       %
\usepackage{amsfonts}       %
\usepackage{nicefrac}       %
\usepackage{microtype}      %
\usepackage{lipsum}
\usepackage{graphicx}
\usepackage{xcolor}
\usepackage{subcaption}
\usepackage{tikz}
\usepackage{wrapfig}
\usetikzlibrary{bayesnet}

\newcommand\independent{\protect\mathpalette{\protect\independenT}{\perp}}
\def\independenT#1#2{\mathrel{\rlap{$#1#2$}\mkern2mu{#1#2}}}
\usepackage{thmtools} 
\usepackage{thm-restate}

\usepackage[ruled,vlined]{algorithm2e}

\usepackage{todonotes}

\usepackage{cleveref}
\crefname{figure}{Fig.}{Figs.}
\crefname{definition}{Defn.}{Defns.}
\crefname{corollary}{Cor.}{Cors.}
\crefname{proposition}{Prop.}{Props.}
\crefname{theorem}{Thm.}{Thms.}
\crefname{remark}{Remark}{Remarks}
\crefname{principle}{Principle}{Principles}
\crefname{lemma}{Lemma}{Lemmata}
\crefname{claim}{Claim}{Claims}
\crefname{table}{Tab.}{Tabs.}
\crefname{section}{\S}{\S\S}
\crefname{subsection}{\S}{\S\S}
\crefname{subsubsection}{\S}{\S\S}
\crefname{assumption}{Asm.}{Asms.}
\crefname{appendix}{Appx.}{Appx.}
\numberwithin{equation}{section}
\usepackage{notation}
\numberwithin{theorem}{section}
\renewcommand{\baselinestretch}{0.995}
\usepackage{enumitem}
\setlist[itemize]{leftmargin=*,itemsep=0em}
\setlist[enumerate]{leftmargin=*,itemsep=0em}
\usepackage[toc,page,header]{appendix}
\usepackage{minitoc}
\usepackage{authblk}

\newcommand{\changelinkcolor}[1]{\hypersetup{linkcolor=#1}}  
\AtBeginDocument{%
  
}

\title{Causal Discovery in Heterogeneous Environments Under the Sparse Mechanism Shift Hypothesis}

\author[1]{Ronan Perry}
\author[1,2]{Julius von K\"ugelgen\thanks{Shared last author.} $\ $ }
\author[1]{Bernhard Sch\"olkopf $^*$}
\affil[1]{Max Planck Institute for Intelligent Systems, T\"ubingen, Germany}
\affil[2]{University of Cambridge, United Kingdom}
\affil[ ]{\texttt{rflperry@uw.edu}, \texttt{jvk@tue.mpg.de}, \texttt{bs@tue.mpg.de}}

\begin{document}
\doparttoc%
\faketableofcontents%

\maketitle
\begin{abstract}
    Machine learning approaches commonly rely on the assumption of independent and identically distributed (\textit{i.i.d.}) data.
    In reality, however, this assumption is almost always violated due to distribution shifts between environments.
    Although valuable learning signals can be provided by heterogeneous data from changing distributions, it is also known that learning under arbitrary (adversarial) changes is impossible.
    Causality provides a useful framework for modeling distribution shifts, since causal models encode both observational and interventional distributions.
    In this work, we explore the \textit{sparse mechanism shift hypothesis}, which posits that distribution shifts occur due to a \emph{small} number of changing causal conditionals.
    Motivated by this idea, we apply it to learning causal structure from heterogeneous environments, where i.i.d.\ data only allows for learning an equivalence class of graphs without restrictive assumptions.
    We propose the \textit{Mechanism Shift Score} (\MSS), a score-based approach amenable to various empirical estimators, which provably identifies the entire causal structure with high probability if the sparse mechanism shift hypothesis holds.
    Empirically, we verify behavior predicted by the theory and compare multiple estimators and score functions to identify the best approaches in practice.
    \looseness-1 Compared to other methods, we show how \MSS\ bridges a gap by both being nonparametric as well as explicitly leveraging sparse changes.
\end{abstract}

\section{Introduction}
Classical machine learning methods and theory presume data to be independently and identically distributed \textit{(i.i.d.)}. Although there has been huge success under this assumption, research on topics including adversarial examples~\citep{szegedy2014,goodfellow2015}, distribution shifts~\citep{quinonero2008dataset, scholkopf2012on, rojas2018invariant, eberhardt2008sufficient}, and  ``spurious'' correlations~\citep{beery2018recognition} has highlighted its fragility. 
Open questions remain as to how we can relax the i.i.d.\ assumption and still learn useful models, since learning under unrestricted adversarial distribution shifts seems infeasible~\citep{david2010impossibility}. Causal models naturally provide structure to a distribution via a factorization into causal \textit{mechanisms}, the processes by which variables are dependent on their direct causes. Hence, they are a natural basis for studying distribution shifts. Based on the idea of the independence of causal mechanisms~\citep{Pearl2000causality,scholkopf2012on}, the \textit{sparse mechanism shift hypothesis}~\citep{scholkopf2021towards} posits that distribution shifts are the result of changes in only a subset of the causal model's mechanisms. This presents a promising relaxation with many potential applications throughout machine learning~\citep{Locatello2020weakly, scholkopf2021towards, Kugelen2021selfsupervised, lippe2022citris,bengio2019meta}.

In particular, causal discovery---the inference of qualitative structure encoded by a graph and underlying causal relationships between variables---is a fundamental scientific problem with broad applications for which distribution shifts have been successfully leveraged.
Classical approaches, which assume i.i.d.\ data from a single environment or domain, can be broadly categorized into three classes. \textit{Constraint-based} methods, such as the PC~\citep{spirtes2001causation} and FCI~\citep{spirtes1999algorithm} algorithms, perform a series of statistical independence tests to identify the existence of causal relationships under various assumptions (e.g., causal faithfulness) and then use certain rules to infer causal directions.
\textit{Score-based} methods, such as GES~\citep{chickering2002optimal}, optimize some score function, such as the penalized likelihood BIC~\citep{schwartz1983bic}, over the set of possible graphs.
These methods all come with asymptotic \textit{structure identifiability} guarantees of recovering the (Markov) equivalence class~\citep{glymour2019review}, but they will rarely uniquely recover the true causal graph. Further work on \textit{functional-form-based} methods, beginning with non-Gaussian assumptions~\citep{shimizu2006linear}, provides various ways to recover a specific DAG by assuming a certain functional form of the causal mechanisms~\citep{hoyer2009nonlinear, zhang2009identifiability, peters2012identifiability, janzing2012information}, at the cost of potential misspecification.

Although identification is possible through actively specified interventions~\citep{eberhardt2005number,eberhardt2007interventions,he2008active,hauser2014two,shanmugam2015learning}, natural distribution shifts on the same variables across different environments allow for promising approaches under a relaxation of the i.i.d.\ assumption~\citep{eaton2007exact, ke2019learning, scholkopf2012on, lagani2012learning, peters2016causal, heinze2018invariant,he2016causal,GonZhaLiuTaoSch16,huang2017behind, ghassami2017learning, ghassami2018multi, huang2020causal, mooij2020joint, guo2022causal, lachapelle2022disentanglement}. Such \textit{multi-environment} methods can learn the direct causes of a specific variable~\citep{peters2016causal, heinze2018invariant} or the entire causal structure~\citep{ghassami2017learning, ghassami2018multi, huang2020causal} without requiring knowledge of which causal mechanisms change. \looseness-1 Yet, as noted in these work but not studied, performance is dependent on which and how many changes occur.

\textbf{Overview and contributions.}
After a review of causal graphical models and formalization of our setting~(\cref{sec:background}), we discuss key related work on multi-environment causal discovery~(\cref{sec:related_work}). We demonstrate that sparse distribution changes provide structure identifiability in the bivariate case, and useful information in multivariate settings~(\cref{sec:leveraging_changes}). Based on the observation that pairwise comparisons of environments better leverage sparse changes~(see~\cref{fig:oracle_sparse_pc}), we propose the \textit{Mechanism Shift Score}, a score-based causal discovery algorithm with theoretical guarantees~(\cref{sec:min_shift_criterion}). Empirical results confirm the theory~(\cref{sec:experiments}). In summary, the present work makes the following contributions:
\begin{itemize}
    \item We prove that by relaxing the i.i.d.\ assumption via the sparse mechanism shift assumption, bivariate causal structure is identifiable without parametric assumptions~(\cref{cor:bivariate_identifiability}).
    \item \looseness-1 We introduce the \textit{Mechanism Shift Score (\MSS)}, defined as the number of conditional distributions implied by a graph which change across all pairs of environments~(\cref{sec:min_shift_criterion}). We prove that the true causal (multivariate) graph minimizes the \MSS\
    over possible graphs~(\cref{prop:min_shift}).%
    \item \looseness-1 We provide rates of convergence showing that with a sufficient number of sparsely changing environments, the causal graph \emph{uniquely} minimizes the score function with high probability~(\cref{cor:identifiability}). Our rates 
    readily apply to existing literature on learning individual mechanisms~\citep{peters2016causal, heinze2018invariant} and the entire graph~\citep{ghassami2018multi, huang2020causal}, where a study of the role of sparsity was previously missing.
    \item We demonstrate empirically that sparsity and pairwise comparisons are useful and show how the \MSS\ accommodates various parametric~\citep{peters2016causal, ghassami2017learning, ghassami2018multi} and nonparametric~\citep{heinze2018invariant, huang2020causal} estimators~(\cref{sec:experiments}).
\end{itemize}

\begin{figure}[t]
    \vspace{-.75em}
    \centering
    \includegraphics[width=\textwidth]{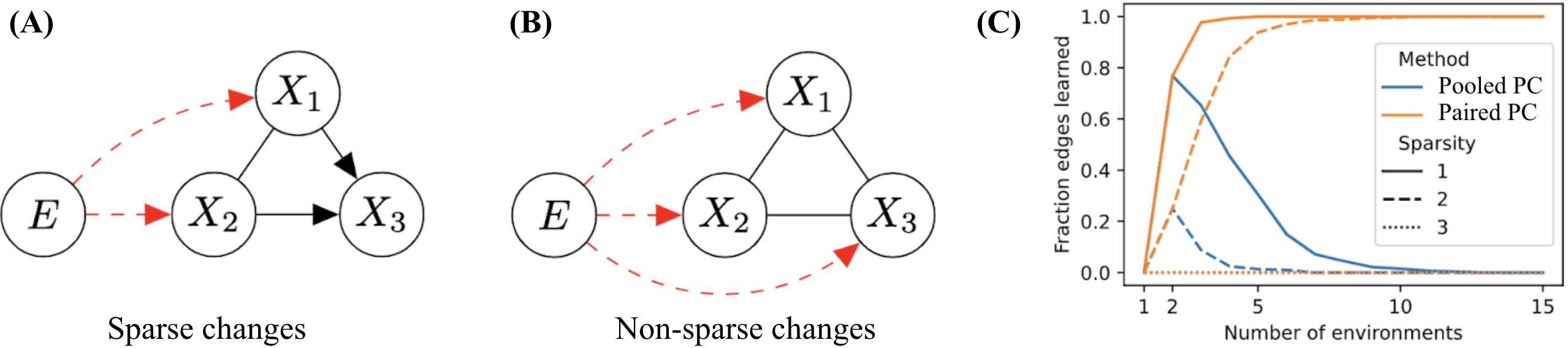}
    \caption{\small \textbf{Sparse shifts yield structure identifiability}. In a causal model with a fully connected DAG over $\Xb:=\{X_1,X_2,X_3\}$, no edge directions can be learned from i.i.d.\ observational data.
    Given multiple datasets on $\Xb$ across different environments with possible distribution shifts, the PC algorithm~\citep{spirtes2001causation} can be applied to the pooled data augmented by the environment index variable $E$~\citep{huang2020causal}. \textbf{(A)}~A mechanism shift $E \to X_i$ can allow orientation of some edges. \textbf{(B)}~Ddense shifts prohibit any orientations~\citep{huang2020causal}.
    \looseness-1 \textbf{(C)}~Combining results learned across \textit{pairs} of pooled environments leads to identifiability under \textit{sparse} shifts. Pooling all environments leads to dense shifts, even if pairs of environments differ sparsely; under \textit{dense} shifts, only the equivalence class is learned.}
    \label{fig:oracle_sparse_pc}
\end{figure}

\section{Problem setting and notation}\label{sec:background}

We start by building up the causal framework needed to understand our work and related literature. It relies on common graph-theoretic terminology which we review for completeness in~\cref{appendix:graph_theory}.

\textbf{Causal terminology.}
Causal relationships between variables are encoded in a causal graphical model (CGM) which links graphical and distributional properties via certain assumptions.

\begin{definition}[Causal Graphical Model (CGM)]
A causal graphical model (CGM) $\Mcal=(G, \PP_\Xb)$ over $d$~random variables $\Xb = \{X_1, ..., X_d\}$ consists of (i) a \textit{directed acyclic graph} (DAG)~$G$ with vertices~$\Xb$ and edges $X_i \to X_j$ iff $X_i$ is a direct cause of $X_j$, (ii) and a joint distribution $\PP_\Xb$ which factorizes (is \textit{Markovian}) over $G$.\footnote{Throughout, we assume the existence of densities with respect to the Lebesgue measure.}
Formally, we have the following \textit{Markov} or \textit{causal factorization}:
\begin{equation}
\label{eq:obs_dist_causal_factorisation}
    \PP_\Xb(X_1, ..., X_d)=\prod_{j=1}^d \PP_\Xb(X_j|\PA_j),
\end{equation}%
\looseness-1 where $\PA_j$ are the parents (direct causes) of 
$X_j$ in $G$
and $\PP_\Xb(X_j|\PA_j)$ is the \textit{causal mechanism} of~$X_j$.
\end{definition}

Implicit in the definition is the assumption that there is \textit{no hidden confounding}, i.e., that any common cause of two or more observed variables is included in $\Xb$; for this reason it is also referred to as \emph{causal sufficiency}. This is a strong assumption to make and important to keep in mind in applications.

\looseness-1 The Markov factorization~\eqref{eq:obs_dist_causal_factorisation} of the CGM encodes various conditional independences between variables. The \textit{Markov equivalence class (MEC)} is the set of DAGs which share the same set of conditional independence relations; graphically, it is the set of DAGs which share the same \textit{skeleton} (set of edges regardless of direction) and \textit{v-structures} ($X_i\to X_j\leftarrow X_k$, but $X_i\not\leftrightarrow X_k$)~\citep[Lem.~6.25]{peters2017elements}. A set of DAGs such as the MEC are commonly represented as a \textit{completed partially directed acyclic graph (CPDAG)} in which edges are directed if directed in all DAGs in the set, and otherwise left undirected~\citep{hauser2012characterization}. Incorrect DAGs in the MEC induce \textit{non-causal factorizations} containing \textit{non-causal conditionals} which differ from the true causal mechanisms. Furthermore, in a CGM, equation~\eqref{eq:obs_dist_causal_factorisation}
is equivalent to the \textit{global Markov condition} which states that for disjoint vertex sets $\Ab,\Bb,\Zb$ in $G$:
\begin{equation}\label{eq:global_Markov_property}
    \Ab \independent_{\!\!G}~\Bb~|~\Zb \implies \Ab \independent \Bb~|~\Zb
\end{equation}
\looseness-1
where $\independent_{\!\!G}$ denotes d-separation in~$G$~(\cref{appendix:graph_theory})~\cite[Thm.~6.22]{peters2017elements}. While the Markov property allows us to derive distributional properties from the DAG, causal discovery concerns deriving graphical properties from distributional properties.  This requires the \textit{causal faithfulness assumption}, effectively assuming that variables are not statistically independent unless implied by the graph.

\begin{assumption}[Causal faithfulness]\label{assmpt:causal_faithfulness}
The observational distribution $\PP_\Xb$ is said to be \textit{faithful} to the causal graph $G$ if every conditional independence relationship in $\PP_\Xb$ implies d-separation in $G$ (i.e., d-connection implies statistical dependence~\citep[Def.~6.33]{peters2017elements}).
\end{assumption}

\textbf{Multi-environment data.}
We assume that we observe a collection $\Dcal$ of datasets from a set $\Ecal$ of (possibly different) environments, where each dataset $\Dcal^e$ from environment $e$ contains $n_e$ independent and identically distributed (i.i.d.) observations from some joint distribution $\PP^e_\Xb$, 
\begin{equation*}
    \Dcal = \{\Dcal^1, ..., \Dcal^{n_{\Ecal}}\}
    \quad
    \mathrm{where}
    \quad
    \Dcal^e=\{\Xb^{e,1},\dots,\Xb^{e,n_e}\}
    \overset{\mathrm{i.i.d.}}{\sim} 
    \PP^e_\Xb
    \quad
    \mathrm{for}
    \quad
    e \in \Ecal = \{1, \dots, n_{\Ecal}\}.
\end{equation*}
Environments can encapsulate experimental settings, or broader contexts such as climate or time~\citep{mooij2020joint}. 
A key assumption of our setting is that environments arise from different ``soft'' interventions on an underlying shared CGM $\Mcal$. Specifically, the CGM $\Mcal$ entails a set of \textit{interventional distributions}, resulting from changing a subset of the \textit{mechanisms} $\PP_\Xb(X_j|\PA_j)$ to some different $\tilde{\PP}_\Xb(X_j|\PA_j)$.%
\footnote{Note that this includes ``hard'' interventions $\tilde{\PP}_\Xb(X_j~|~\PA_j)=\delta(X_j - x)$ which fix $X_j$ to a specific value $x$, and also ``soft'' interventions which merely alter the functional relationship.} 

\begin{assumption}[Shared mechanisms]\label{assumption:shared_mechs}
Each environment $e$ independently results from $\Mcal$ by intervening on an (unknown) subset $\Ical^e\subseteq[d]$ of mechanisms, i.e.,
$\tilde{\PP}^e_\Xb$ 
can be written as
\begin{equation}
    \tilde{\PP}_\Xb^e(X_1, ..., X_d) = 
    \left(\prod_{j\in \Ical^e} \tilde{\PP}^e_\Xb(X_j~|~\PA_j)\right)
    \prod_{j\in[d]\setminus \Ical^e}\PP_\Xb(X_j~|~\PA_j).
\end{equation}
\end{assumption}
We make the following common assumption about how these changes arise.
\begin{assumption}[Independent causal mechanisms (ICM)~\cite{Pearl2000causality, scholkopf2012on}]\label{assmpt:independent_mechanisms}
A change in $\PP(X_j~|~\PA_j)$ has no effect on and provides no information on $\PP(X_k~|~\PA_k)$ for any $k \neq j$.
\end{assumption}

Since causal mechanisms are fixed within an environment and potentially vary across them, we can create an \textit{augmented CGM} to unify CGMs from different environments. Note that in~\cref{assumption:shared_mechs} and the following definition, the same causal parents and DAG are preserved over different environments; the distribution changes are limited to soft interventions which do not change a variable's causal~parents.

\begin{definition}[Augmented CGM~\cite{huang2020causal}]\label{def:aug_cgm}
Let $\{(G, \PP_\Xb^e)\}_{e=1}^{n_{\Ecal}}$ be a collection of CGMs over the DAG $G=(\Xb, T)$ from multiple environments. The augmented CGM is defined as $\Mcal':=(G', \PP_{\Xb \cup E})$ where (i) $E$ is a random environment indicator variable with support $\Ecal$, and (ii) the augmented DAG $G'$ has vertices $\Xb \cup E$ and edge set $T \cup \{(E, X_j): \exists e, e' \in \Ecal \text{ s.t. } \PP^e_\Xb(X_j|\PA_j^G) \neq \PP^{e'}_\Xb(X_j|\PA_j^G)\}$.
\end{definition}

Note that $\PP_{\Xb \cup E}$ is Markovian to $G'$, inheriting the factorization from the underlying CGMs along with the added dependence on $E$. With respect to the augmented DAG, \cref{assmpt:independent_mechanisms} implies that existence of the edge $E \to X_i$ provides no information on the existence of an edge $E \to X_j$ for~$j \neq i$. As discussed by~\citet{huang2020causal}, since $E$ can be a common cause of variables in $\Xb$, causal sufficiency in the original CGM over $\Xb$ is violated. We instead must assume \textit{pseudo-causal sufficiency}. See~\cref{appendix:details} for further discussion of this assumption and the ICM principle.

\begin{assumption}[Pseudo causal sufficiency~\cite{huang2020causal}]\label{assumption:pseudo_suff}
Any unobserved confounders of variables in $\Xb$ can be written solely as functions of $E$. 
Thus, within any given environment $e$, all unobserved confounders are fixed and causal sufficiency holds.
\end{assumption}

\textbf{Sparse mechanism shift (SMS) hypothesis.}
Another key assumption, supported by~\cref{assmpt:independent_mechanisms}'s implication that a change to one mechanism does not imply changes to others, is the following:

\begin{assumption}[SMS~\citep{scholkopf2021towards}]\label{assmpt:sparse_change}
\looseness-1 Changes in mechanisms between observed environments are sparse:
\begin{equation}
    \textstyle
    0 < |\Ical^e| < d
\end{equation}
\end{assumption}

The value of this assumption when met is illustrated in~\cref{fig:oracle_sparse_pc} and will be elaborated upon in~\cref{sec:min_shift_criterion}.

\section{Related work on causal discovery from multiple environments}\label{sec:related_work}
\looseness-1 
Causal discovery from changing distributions and causal mechanisms has a long history, going back to Simon's {\em invariance criterion}, stating that the true causal order is the one that is invariant under the right sort of intervention~\citep{hoover90,Hoover06}. \citet{tian2001causal} infer a causal order by testing which marginal distributions change under a single intervention. Invariant causal prediction (ICP)~\citep{peters2016causal} can identify the causal parents of a target variable under the assumption that the target's causal mechanism is invariant across environments~\citep{scholkopf2012on,peters2016causal, heinze2018invariant}. However, applying ICP to each variable in order to learn all sets of causal parents and hence the causal graph is not immediately admissible: the invariance assumption would imply that all variables are invariant, and thus there are no mechanism changes to learn from.

\looseness-1 Learning the MEC from i.i.d.\ data is a well-studied problem, and under certain assumptions, essentially solved~\citep{spirtes1999algorithm, spirtes2001causation, chickering2002optimal}. \rebuttal{In our multi-environment setting, \mbox{\citet{yang_characterizing_2019}} characterized the $\Ical$-MEC, a subset of the MEC, which was subsquently generalized to a $\Psi$-MEC~\mbox{\citep{jaber_causal_2020}} under an unknown observational environment. \mbox{\citet{brouillard_differentiable_2020}} developed an estimator under unknown intervention targets, but none of these works implement an algorithm or experiments on soft and unknown interventions, and it is still an open question how best to learn the true causal DAG from the MEC.} \citet{ghassami2017learning} apply ideas from linear ICP to identify the causal DAG. Assuming \textit{linear} causal mechanisms in which only the noise distributions change, they compare pairs of environments and orient edges to meet this requirement. \citet{ghassami2018multi} allow for any kind of change within a \textit{linear} model. Based on the ICM principle, they demonstrate that non-causal DAGs induce a larger number of changes in the linear model parameters  than the causal DAG. By counting parameter changes across pairwise environments, they can determine the causal order of variables. \citet{huang2020causal} remove any functional restrictions through a two-stage approach. First, they use the PC algorithm on the augmented CGM, pooling all the data; this both identifies the MEC but more importantly can also orient additional edges. Since not all edges are guaranteed to be oriented, they propose a second stage, relying on a novel measure of mechanism dependence to individually orient remaining edges.

\looseness-1 A consistent yet sparingly explored theme across all of these methods is the \rebuttal{identifiability of the true DAG, rather than an equivalence class, and} impact of the sparsity of changes across environments. \rebuttal{Although both \mbox{\citet{yang_characterizing_2019}} and \mbox{\citet{jaber_causal_2020}} characterize and identify subsets of the MEC, requirements on the identifiability of the true DAG are unclear.} In ICP, it is briefly noted that the method is applicable when the target variable experiences sparse shifts, as long as we assume a maximum degree of sparsity~\citep[\S 6.2]{peters2016causal}. Although not discussed, the performance of methods from~\citet{ghassami2017learning,ghassami2018multi} also depend on sparsity. As mentioned by \citet{huang2020causal}, identifiability requires the invariance of some mechanisms ad their pooled PC approach cannot identify edge directions when both adjacent variables change. We will provide clarity on identifiability via sparsity.

\section{Leveraging sparse mechanism changes for causal discovery}\label{sec:leveraging_changes}
\looseness-1
The augmented DAG is a powerful tool for understanding the implications of changing causal mechanisms. We build up intuition in the bivariate setting and then consider the general multivariate setting. From now on, we assume causal faithfulness on the augmented CGM.

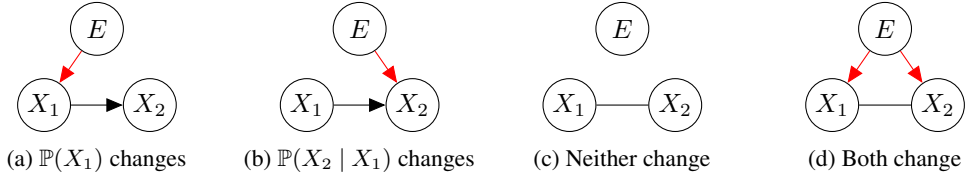
\begin{figure}[t]
    \centering
    \begin{subfigure}{0.25\textwidth}
        \centering
        \begin{tikzpicture}
        \newcommand{\xshift}{2em}
        \newcommand{\yshift}{2em}    
            \centering
            \node (E) [latent] {$E$};
            \node (PA) [latent, below=of E, xshift=-\xshift, yshift=\yshift] %
            {$X_1$};
            \node (X) [latent, below=of E, xshift=\xshift, yshift=\yshift] %
            {$X_2$};
            \edge{PA}{X};
            \edge[red]{E}{PA};
        \end{tikzpicture}
        \caption{$\PP(X_1)$ changes}
        \label{fig:prop11_proof_a}
    \end{subfigure}%
    \begin{subfigure}{0.25\textwidth}
        \centering
        \begin{tikzpicture}
        \newcommand{\xshift}{2em}
        \newcommand{\yshift}{2em}    
            \centering
            \node (E) [latent] {$E$};
            \node (PA) [latent, below=of E, xshift=-\xshift, yshift=\yshift] %
            {$X_1$};
            \node (X) [latent, below=of E, xshift=\xshift, yshift=\yshift] %
            {$X_2$};
            \edge{PA}{X};
            \edge[red]{E}{X};
        \end{tikzpicture}
        \caption{$\PP(X_2~|~X_1)$ changes}
        \label{fig:prop11_proof_b}
    \end{subfigure}%
    \begin{subfigure}{0.25\textwidth}
        \centering
        \begin{tikzpicture}
        \newcommand{\xshift}{2em}
        \newcommand{\yshift}{2em}    
            \centering
            \node (E) [latent] {$E$};
            \node (PA) [latent, below=of E, xshift=-\xshift, yshift=\yshift] %
            {$X_1$};
            \node (X) [latent, below=of E, xshift=\xshift, yshift=\yshift] %
            {$X_2$};
            \edge[-]{PA}{X};
        \end{tikzpicture}
        \caption{Neither change}
        \label{fig:prop11_no_changes}
    \end{subfigure}%
    \begin{subfigure}{0.25\textwidth}
        \centering
        \begin{tikzpicture}
        \newcommand{\xshift}{2em}
        \newcommand{\yshift}{2em}    
            \centering
            \node (E) [latent] {$E$};
            \node (PA) [latent, below=of E, xshift=-\xshift, yshift=\yshift] %
            {$X_1$};
            \node (X) [latent, below=of E, xshift=\xshift, yshift=\yshift] %
            {$X_2$};
            \edge[red]{E}{PA};
            \edge[red]{E}{X};
            \edge[-]{PA}{X}
        \end{tikzpicture}
        \caption{Both change}
        \label{fig:prop11_both_change}
    \end{subfigure}%
    \caption{\small (a, b) Visualizations of the two cases explored in the proof of \cref{prop:bivariate_changes}. In both cases, $X_2 \not\independent E$ unconditionally and $X_1 \not\independent E~|~X_2$ directly or by conditioning on the potential collider $X_2$. (c, d) The two other possible situations; neither allow us to orient the edge between $X_1$ and $X_2$.}
    \label{fig:prop11_proof}
\end{figure}

\textbf{Bivariate case.}
Given a causal model composed of two associated variables, identifiability of the causal direction requires assumptions such as the functional form or existence of targeted interventions~\citep[\S 4.1]{peters2017elements}. We show how sparsely changing mechanisms also provide identifiability.

\begin{proposition}[Both non-causal conditionals change]\label{prop:bivariate_changes}
\looseness-1 Consider the bivariate setting $X_1\to X_2$. If either causal mechanism $\PP(X_2~|~X_1)$ or $\PP(X_1)$ changes, then both $\PP(X_1~|~X_2)$ and $\PP(X_2)$ change.%
\end{proposition}

\begin{corollary}[Bivariate identifiability]\label{cor:bivariate_identifiability}
In the setting of \cref{prop:bivariate_changes}, if only one causal mechanism changes (sparsity), then the bivariate causal structure is identifiable.
\end{corollary}%

\begin{proof}(\cref{prop:bivariate_changes})
\looseness-1 We consider each case separately and use a proof by contradiction (of faithfulness).

    (i) If $\PP(X_1)$ changes (see \cref{fig:prop11_proof_a}), then $G_{\Xb \cup E}$ contains the edge $E \to X_1$. If $\PP(X_2)$ remained invariant, then $X_2 \independent E$ (unfaithful due to the unblocked path $E\to X_1\to X_2$). If $\PP(X_1~|~X_2)$ remained invariant, then $X_1 \independent E~|~X_2$ (unfaithful due to the direct path $E\to X_1$).

    (ii) If $\PP(X_2~|~X_1)$ changes (see \cref{fig:prop11_proof_b}), then $G_{\Xb \cup E}$ contains the edge $E \to X_2$. If $\PP(X_2)$ remained invariant, then $X_2 \independent E$ (unfaithful due to the direct path $E\to X_2$). If $\PP(X_1~|~X_2)$ remained constant, then $X_1 \independent E~|~X_2$ (unfaithful due to the unblocked collider path $E\to X_2 \leftarrow X_1$).
\end{proof}
\begin{proof}(\cref{cor:bivariate_identifiability})
If either causal mechanism changes, by~\cref{prop:bivariate_changes} both conditionals in the non-causal factorization change. Hence, the causal structure is the one with only one mechanism change.
\end{proof}
\textbf{Multivariate case.}
Non-causal conditional distributions of $X_j$ may change across environments even if the causal mechanism $\PP(X_j~|~\PA_j)$ does not change.
This occurs if the conditioning set leaves open a dependence between $E$ and $X_j$ in~$G_{\Xb \cup E}$.

\begin{restatable}{lemma}{lemFlip}\label{lem:edge_flip}
For any $X_j \in \Xb$  and set $\Zb \subseteq \Xb\setminus\{X_j\}$, the conditional distribution $\PP(X_j~|~\Zb)$ changes if and only if the following d-connectedness relationship holds:
\begin{equation*}
X_j \not\independent_{\!\!G_{\Xb \cup E}}~E~|~\Zb\,.
\end{equation*}%
\end{restatable}%
The result follows from the Markov property and faithfulness; for all complete proofs, see~\cref{appendix:full_proofs}.

Since the Markov equivalence class is relatively easily available and thus often the starting point of open questions in causal discovery, we specify the implications of \cref{lem:edge_flip} in this setting. Note that due to the shared skeleton of all DAGs in the equivalence class, the conditioning set for any $X_j$ in any DAG in the MEC is a subset of $X_j$'s true parents and children $\PA_j^G \cup \CH_j^G$~\citep{spirtes2001causation, lachapelle2022disentanglement}.

\begin{figure}[t]
    \vspace{-0.75em}
    \centering
    \captionsetup[subfigure]{labelformat=empty}
    \newcommand{\xshift}{2em}
    \newcommand{\yshift}{2.5em}  
    \begin{subfigure}{0.3\textwidth}
        \centering
        \begin{tikzpicture}
            \centering
            \node (X) [latent] {$X_j$};
            \node (PA) [latent, above=of X, xshift=-\xshift, yshift=-\yshift, inner sep=0em] 
            {\scriptsize$\PA_j^{\setminus \Zb}$};
            \node (PAZ) [latent, above=of X, xshift=\xshift, yshift=-\yshift, fill={rgb:black,1;white,8}, inner sep=0em]
            {\scriptsize$\PA_j^{\cap \Zb}$};
            \node (CH) [latent, below=of X, xshift=\xshift, yshift=\yshift, inner sep=0em] 
            {\scriptsize$\CH_j^{\setminus \Zb}$};
            \node (CHZ) [latent, below=of X, xshift=-\xshift, yshift=\yshift, fill={rgb:black,1;white,8}, inner sep=0em]
            {\scriptsize$\CH_j^{\cap \Zb}$};
            \node (E) [latent, left=of X, xshift=-0.5\xshift]
            {$E$};
            \edge[]{PAZ}{X};
            \edge[]{PA}{X}
            \edge[]{X}{CHZ};
            \edge[]{X}{CH};
            \edge[red]{E}{X};
        \end{tikzpicture}
        \caption{(i): a direct cause}
    \end{subfigure}%
    \begin{subfigure}{0.3\textwidth}
        \centering
        \begin{tikzpicture}
            \centering
            \node (X) [latent] {$X_j$};
            \node (PA) [latent, above=of X, xshift=-\xshift, yshift=-\yshift, inner sep=0em] 
            {\scriptsize$\PA_j^{\setminus \Zb}$};
            \node (PAZ) [latent, above=of X, xshift=\xshift, yshift=-\yshift, fill={rgb:black,1;white,8}, inner sep=0em]
            {\scriptsize$\PA_j^{\cap \Zb}$};
            \node (CH) [latent, below=of X, xshift=\xshift, yshift=\yshift, inner sep=0em] 
            {\scriptsize$\CH_j^{\setminus \Zb}$};
            \node (CHZ) [latent, below=of X, xshift=-\xshift, yshift=\yshift, fill={rgb:black,1;white,8}, inner sep=0em]
            {\scriptsize$\CH_j^{\cap \Zb}$};
            \node (E) [latent, left=of X, xshift=-0.5\xshift]
            {$E$};
            \edge[]{PAZ}{X};
            \edge[]{PA}{X}
            \edge[]{X}{CHZ};
            \edge[]{X}{CH};
            \edge[red,dashed]{E}{PA};
        \end{tikzpicture}
        \caption{(ii): an unblocked ancestor}
    \end{subfigure}%
    \begin{subfigure}{0.34\textwidth}
        \centering
        \begin{tikzpicture}
            \centering
            \node (X) [latent] {$X_j$};
            \node (PA) [latent, above=of X, xshift=-\xshift, yshift=-\yshift, inner sep=0em] 
            {\scriptsize$\PA_j^{\setminus \Zb}$};
            \node (PAZ) [latent, above=of X, xshift=\xshift, yshift=-\yshift, fill={rgb:black,1;white,8}, inner sep=0em]
            {\scriptsize$\PA_j^{\cap \Zb}$};
            \node (CH) [latent, below=of X, xshift=\xshift, yshift=\yshift, inner sep=0em] 
            {\scriptsize$\CH_j^{\setminus \Zb}$};
            \node (CHZ) [latent, below=of X, xshift=-\xshift, yshift=\yshift, fill={rgb:black,1;white,8}, inner sep=0em]
            {\scriptsize$\CH_j^{\cap \Zb}$};
            \node (E) [latent, left=of X, xshift=-0.5\xshift]
            {$E$};
            \edge[]{PAZ}{X};
            \edge[]{PA}{X}
            \edge[]{X}{CHZ};
            \edge[]{X}{CH};
            \edge[red, dashed]{E}{CHZ};
        \end{tikzpicture}
        \caption{(iii): a conditioned child's ancestor}
    \end{subfigure}%
    \caption{\small The three possible cases from \cref{cor:edge_flip} with a d-connecting path from $E$ to~$X_j$, conditioned on a subset of neighbors (colored in grey) in the MEC; these neighbors are a subset of the true parents and children.}
    \label{fig:3cases_viz}
\end{figure}
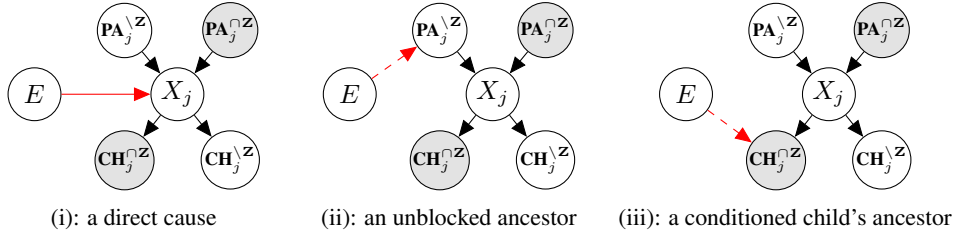

\begin{restatable}{corollary}{corFlip}\label{cor:edge_flip}
For any variable $X_j \in \Xb$  and set $\Zb \subseteq (\PA_j^G \cup \CH_j^G)$ in the augmented graph, the conditional distribution $\PP(X_j~|~\Zb)$ changes if and only if at least one of the following holds:

    (i) $E \to X_j$ [a direct cause].
    
    (ii) $\exists W_\PA \in \PA_j^G \setminus \Zb$ such that
    $W_\PA \not\independent_{\!\!G_{\Xb \cup E}}~E~|~\Zb$
    [unblocked path to unconditioned parent].
    
    (iii) $\exists W_\CH \in \CH_j^G \cap \Zb$ such that $W_\CH \not\independent_{\!\!G_{\Xb \cup E}}~E~|~\Zb\setminus W_\CH$
    [unblocked path to conditioned child].
\end{restatable}

\begin{proof}[Proof sketch]
These cases are visualized in~\cref{fig:3cases_viz}. The global Markov property and faithfulness assumption allow us to interchange d-connection and a distributional change. In the forward direction, a changing mechanism implies a d-connecting path which is necessarily captured in one of the three cases. In the reverse direction, each case opens a d-connecting path between $X_j$ and $E$.
\end{proof}
\vspace{-0.5em}
A direct result of changing causal mechanisms and conditional distributions follows.
\begin{restatable}{corollary}{corShif}\label{cor:true_shift}
Let $G^*$ denote the true (unaugmented) DAG and let $G$ be any other DAG over the same variables. For any $X_j \in \Xb$, a change in $\PP(X_j~|~\PA_j^{G^*})$ implies a change in $\PP(X_j~|~\PA_j^{G})$.
\end{restatable}
\vspace{-0.5em}
\begin{proof}
Under the true causal parents given by $G^*$, a change in mechanism occurs if and only if $E \to X_j$ in $G^*_{\Xb \cup E}$. So, $X_j$ is d-connected to $E$, no matter the conditioning set $\Zb \subset \Xb\setminus\{X_j\}$. Thus, by \cref{lem:edge_flip}, $\PP(X_j~|~\PA_j^{G})$ necessarily changes.
\end{proof}

\section{Causal discovery via the Mechanism Shift Score (\MSS)}\label{sec:min_shift_criterion}
We have established that changing mechanisms provide useful information and by \cref{cor:bivariate_identifiability} can provide identifiability in the bivariate case under faithfulness and sparse changes. We now study identifiability in more general graphs along with approaches to developing practical estimators. \looseness-1 Motivated by the novel discovery of the value of comparing pairwise environments if changes are sparse, we propose the \textit{Mechanism Shift Score (\MSS)} estimand with useful theoretical guarantees over DAGs in the MEC.

\textbf{The \MSS\ estimand.}
Given a set~$\Gcal$ of candidate DAGs over the same $d$ variables $\Xb = \{X_1, ..., X_d\}$ with data sampled from $n_{\Ecal}$ distributions $\Pcal := \{\PP_\Xb^e\}_{e=1}^{n_{\Ecal}}$, we count the number of changing conditional distributions in a graph $G \in \Gcal$ by defining
\begin{equation*}\label{eq:min_shift_estimand}
    \MSS_j(G; \Pcal) = \sum_{e'>e}^{n_{\Ecal}} \II\left[\PP^e(X_j|\PA_j^G) \neq \PP^{e'}(X_j|\PA_j^G)\right]
    \quad
    \text{and}
    \quad
    \MSS(G; \Pcal) = \sum_{j=1}^d \MSS_j(G; \Pcal).
\end{equation*}
$\MSS_j(G; \Pcal)$ is the number of pairs of environments in which the conditional distribution of $X_j$ in~$G$ changes; $\MSS(G; \Pcal)$ is the total number of changes across all variables and pairs of environments according to the Markov factorization implied by~$G$. It follows from \cref{cor:true_shift} that the true DAG $G^*$ minimizes (not necessarily uniquely) the number of changing conditionals among all DAGs. 
\begin{proposition}\label{prop:min_shift}
Let $G^*$ be the true DAG in the set~$\Gcal$ of DAGs. Then for all $G \in \Gcal$ and $j \in \{1,\dots,d\}$
\begin{equation*}
    \MSS_j(G^*; \Pcal) \leq \MSS_j(G; \Pcal)
    \qquad\text{and thus}\qquad
    \MSS(G^*; \Pcal) \leq \MSS(G; \Pcal).
\end{equation*}
\end{proposition}
\vspace{-0.5em}
\begin{proof}
By \cref{cor:true_shift}, any change in $\PP(X_j|\PA_j^{G^*})$ implies a change in $\PP(X_j|\PA_j^{G})$ for any other DAG~$G$. Thus, any change counted by $\MSS_j$ on the true DAG will also be detected in every other DAG and so both lower bounds hold.
\end{proof}
\vspace{-0.5em}
\cref{prop:min_shift} can be viewed as the generalization of the Principle of Minimal Changes~\citep{ghassami2018multi}, allowing for mechanism changes beyond the parametric restrictions of a linear model. Identifiability, however, requires us to establish a discerning aspect of the true structure. Recall that the Markov equivalence class~$\Gcal_{\MEC}$ is identifiable. We define the subset
\begin{equation*}
    \Gcal_{\MEC}^{\min} := \argmin_{G \in \Gcal_{\MEC}} \, \MSS(G; \Pcal)
\end{equation*}
of DAGs with minimum \MSS, here defined as the number of mechanism shifts across environments. \rebuttal{It turns out that $\Gcal_{\MEC}^{\min}$ is a $\Psi$-MEC as characterized by \mbox{\citet{jaber_causal_2020}}, although we leave the details and proof of that to~\mbox{\cref{appendix:full_proofs}} and focus our results on identifiability of the true DAG.} In practice, we may employ any generic conditional test for change in mechanism or choose to use a ``softer'' score (e.g., based on \textit{p}-values) to quantify changes along a continuous spectrum. \rebuttal{This is a similar idea to that of~\mbox{\citet{brouillard_differentiable_2020}} who propose a multi-environment likelihood-based approach.}

\cref{prop:min_shift} implies that $G^* \in \Gcal_{\MEC}^{\min}$. \looseness-1 Using probabilistic assumptions based on the idea of sparse changes, we show that, given enough environments, the causal parents and full DAG are identifiable.

\begin{restatable}[Identifiability of causal parents]{lemma}{lemConsistency}\label{lem:consis_mech}
Let $G^*$ be the true DAG in the MEC $\Gcal_{\MEC}$ and $\rho_i$ the probability that the causal mechanism of $X_i$ is different across any two environments. Under~\cref{assmpt:causal_faithfulness,assumption:shared_mechs,assumption:pseudo_suff,assmpt:independent_mechanisms},
for any $j \in \{1,\dots,d\}$, graph $G \in \Gcal_{\MEC}$ such that $\PA_j^{G^*} \neq \PA_j^G$, and lower and upper bounds on the shift probabilities $\rho_i^{\lb} \leq \rho_i \leq \rho_i^{\ub}$ for all $i$, we have that
\begin{equation*}
    \Pr[\MSS_j(G^*; \Pcal) < \MSS_j(G; \Pcal)] \geq 1 - \left(1 - (1 - \rho_j^{\ub})\min_{i
    } \rho_i^{\lb}\right)^{\lfloor n_{\Ecal} / 2\rfloor}.
\end{equation*}
\end{restatable}
\vspace{-0.5em}
\begin{proof}[Proof sketch]
\looseness-1 
From \cref{cor:true_shift}, we know: $\MSS_j(G^*; \Pcal) \leq \MSS_j(G; \Pcal)$. We use the shared skeleton property of all DAGs in a MEC and \cref{cor:edge_flip}, for which only case (i) admits a changing mechanism in the true DAG. \looseness-1 Based on the ICM principle, we examine sufficient conditions in a pair of environments and establish a probabilistic upper bound on all pairs.
\end{proof}

Note that this bound is independent of the other DAG~$G$, so long as $G$ is in the MEC and has a different set of causal parents. As a special case of~\cref{lem:consis_mech}, invariance of the mechanism of $X_j$ implies $\rho_j^{\ub}=0$ and hence provides a bound relevant to invariant causal prediction~\cite{peters2016causal, heinze2018invariant}. Building off of \cref{lem:consis_mech}, we can provide a probabilistic bound on identifiability of the whole graph.

\begin{restatable}[Identifiability of the graph]{theorem}{thmConsistency}\label{thm:consi_dags}
Let $G^*$ be the true DAG in the MEC $\Gcal_{\MEC}$ and $\rho_j$ the probability that the causal mechanism of $X_j$ is different across any two environments. Under assumptions
\ref{assmpt:causal_faithfulness},
\ref{assumption:shared_mechs}, ~\ref{assmpt:independent_mechanisms}, and
\ref{assumption:pseudo_suff},
and bounds $\rho_i^{\lb} \leq \rho_i \leq \rho_i^{\ub}$ for all $i$, we have that
\begin{equation*}
    \Pr[\Gcal_{\MEC}^{\min} = \{G^*\}] \geq %
    1 - |\Gcal_{\MEC}|\left(1 - (1 - \min_i \rho_i^{\ub}) \min_{i} \rho_i^{\lb}\right)^{\lfloor n_{\Ecal} / 2\rfloor}.
\end{equation*}
\end{restatable}
\begin{proof}[Proof sketch]
From \cref{prop:min_shift}, $G^*$ is always in $G^{\min}_{\MEC}$. For each DAG, we use \cref{lem:consis_mech} to bound the probability that all mechanisms exhibit the same number of changes. Then we apply the union bound to establish an upper bound across all DAGs.
\end{proof}
\begin{corollary}\label{cor:identifiability}
If $\rho_i$ is bounded away from 0 and 1 for all $i$,
(a probabilistic form of~\cref{assmpt:sparse_change}),
\begin{equation*}
    \Pr[\Gcal_{\MEC}^{\min} 
    = \{G^*\}] \to 1 \quad \text{as} \quad n_{\Ecal} \to \infty
\end{equation*}
That is, with enough environments we can recover the true DAG from the Markov equivalence class.
\end{corollary}
\vspace{-0.5em}
\begin{proof}
The assumption of bounded probability implies that $\rho_i^{\ub} < 1$ and $\rho_i^{\lb} > 0$ for all $i$.
Hence, by the rate established in \cref{thm:consi_dags}, identifiability is achieved in the limit.
\end{proof}%
\textbf{The \MSS\ estimator.}
An empirical \MSS\ estimator tests if two conditional distributions change across two environments. This can be done using conditional independence tests or equality of distribution tests~\citep{panda2021nonpar}. Under parametric assumptions, models may be fit for each mechanism, and these parameters can then be tested across environments~\citep{ghassami2017learning, ghassami2018multi, peters2016causal}. \citet{heinze2018invariant} provided a comprehensive study of such tests and their power for ICP. More recent but less studied work by \citet{park2021conditional} has provided a kernel-based approach with strong guarantees. In practice care must be taken when using equality of conditional distribution tests, especially if any of their assumptions are violated~\citep{shah2018hardness}.

\textbf{Computational complexity.}
\looseness-1 The score function $\MSS(G; \Pcal)$ is \textit{decomposable}~\citep{hauser2012characterization} in  that it is the sum of local scores $\MSS_j(G; \Pcal)$. In the most na\"ive approach, each mechanism in each of $|\Gcal_\MEC|$ DAGs must be tested across all pairwise environments, on the overall order of $O(|\Gcal_\MEC|dn_{\Ecal}^2)$, without accounting for the complexity of the statistical test.
This can be slow, but if the test runtime scales with sample size $n$ faster than $O(n^2)$, then pairwise tests may actually be faster than pooling the data. In practice, the decomposable property permits a speedup: since many mechanisms will be shared across DAGs, we can test each unique mechanism and then select the results relevant for each DAG. 
Experimentally, we find the main bottlenecks to be large sample sizes and numbers of environments.

\begin{figure}[t]
    \centering
    \includegraphics[width=\textwidth]{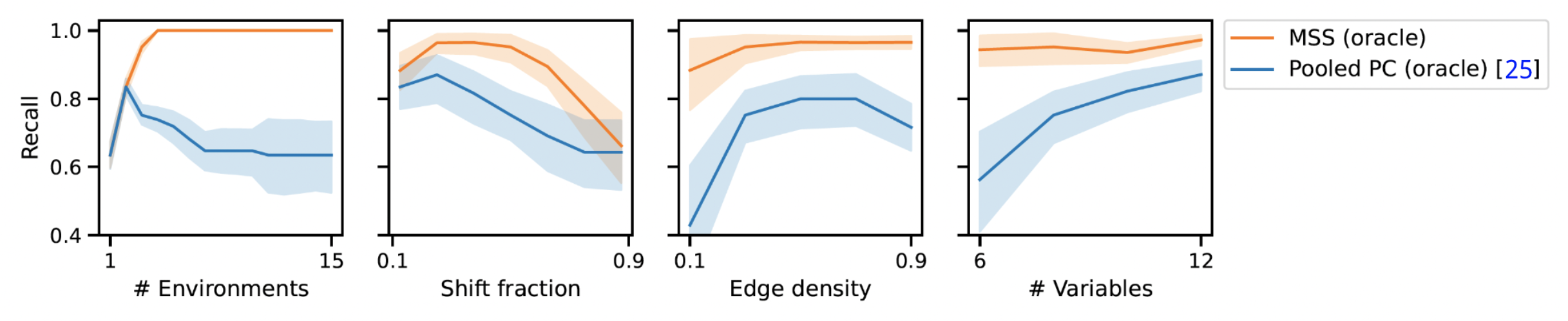}
    \caption{\small \textbf{Oracle rates match the theory.} From left to right (a-d): (a)~With sufficiently many environments, our \MSS\ approach learns the true DAG while pooled PC recovers only the original MEC. (b)~Pairwise comparisons are less beneficial when shifts are extremely sparse or dense. Differences across the two methods are particularly pronounced both in (c)~sparse and dense DAGs, as well as in (d)~smaller DAGs. Shaded regions denote 95\% confidence intervals calculated from bootstrapped data.}
    \label{fig:sim_oracle_rates}
\end{figure}

\section{Structure learning experiments}\label{sec:experiments}
Having established the theoretic value of the \MSS\ estimand, we seek to understand \textit{empirically}: (i) how the oracle \MSS\ and pooled PC approaches compare across experimental settings, (ii) which \MSS\ estimators perform best, and (iii) how the \MSS\ compares to related approaches.\footnote{All code and experiments are available at \href{https://github.com/rflperry/sparse\_shift}{https://github.com/rflperry/sparse\_shift}} 
For ease of comparison, we adapt the simulation setup of~\citet{huang2020causal}. Random DAGs are sampled using an Erd\H{o}s-R\`enyi model~\citep{erdHos1960evolution} in which each edge has some fixed probability of existing (the \textit{edge density}). In each environment, a random set of variables experience a mechanism change according to a fixed number or fraction of shifts. Each variable $j$ in environment $e$ has a randomly sampled mechanism
\begin{equation}
    \label{eq:sim_dist}
    \textstyle
    X_j^e := \sum_{i \in \PA_j} b_{ji}^e f_{ji} (X_i^e) + \sigma_j^e \epsilon_j^e
\end{equation}
where $b_{ji} \sim \Ucal(0.5, 2.5)$, $\sigma_j^e \sim \Ucal(1, 3)$, and $\epsilon_j^e \sim \Ncal(0, 1)$ or $\Ucal(1, 3)$ with equal probability. The functions $f_{ji}$ are selected uniformly at random from $\{x^2, x^3, \tanh, \mathrm{sinc}\}$. Mechanisms in an unobserved baseline environment are sampled and $\sigma_{ji} = 1$ is fixed. \looseness-1 Each observed environment inherits the baseline distributions and mechanisms shifts are resampled per~\eqref{eq:sim_dist}.

\looseness-1 We evaluate the quality of an estimated CPDAG against the true DAG via \textit{precision} and \textit{recall}~\citep{ghassami2017learning, ghassami2018multi, huang2020causal}. Precision is the fraction of \textit{directed} edges in the CPDAG which are correctly oriented. Recall is the fraction of \textit{all} edges in the CPDAG which are oriented. Thus, the true DAG has perfect precision and recall. Since all methods start from the MEC, there are no incorrect edges, only incorrect orientations. \rebuttal{For oracle methods, the precision is perfect and so we only report the recall.}

\begin{figure}[t]
    \centering
    \includegraphics[width=\textwidth]{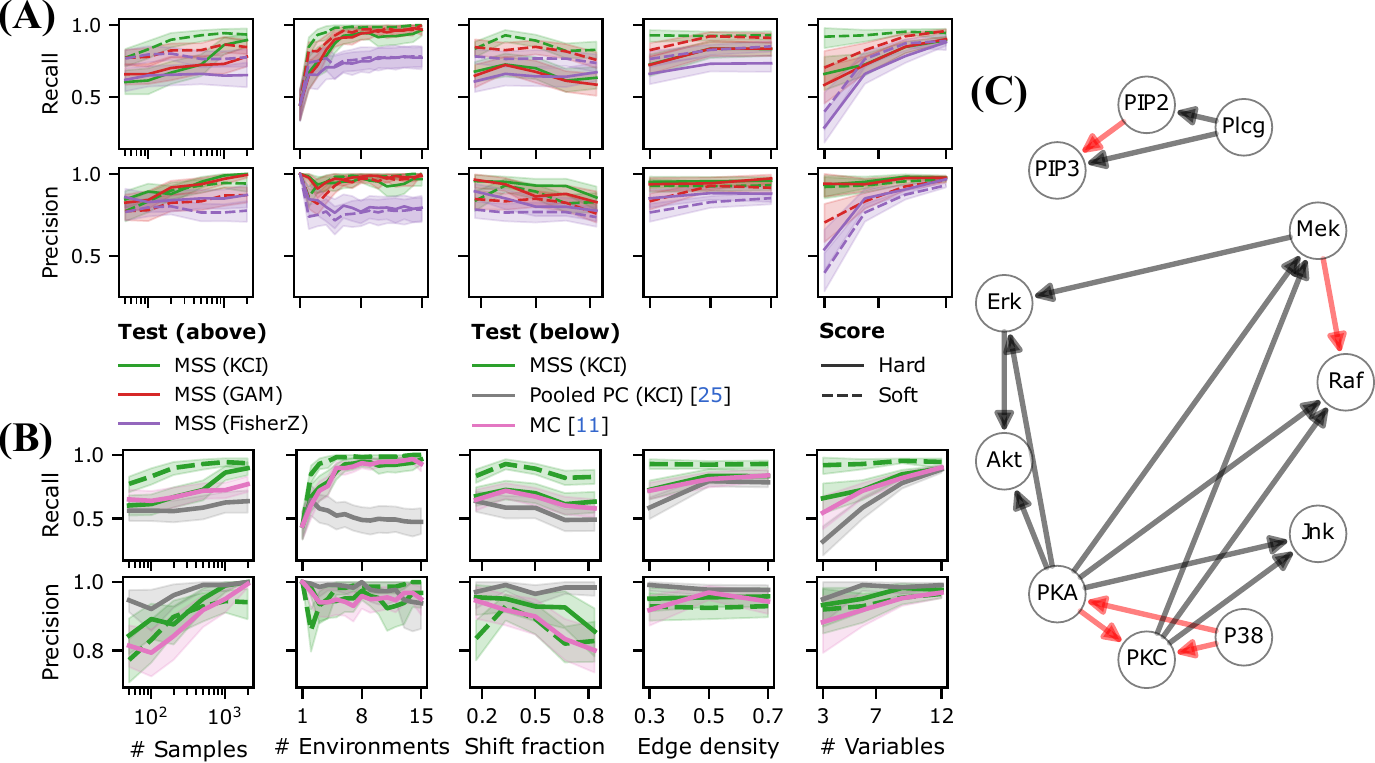}
    \caption{\small \textbf{(A)} Nonparametric hypothesis tests perform well in nonlinear simulations, and soft scores succeed. Notably, recall converges with increasing environments. KCI appears to best balance high recall and precision. \textbf{(B)} In simulation, pairwise approaches improve with more environments, unlike Pooled PC. Although the parametric MC works surprisingly well across settings, the nonparametric \MSS\ with KCI has superior precision and recall. \textbf{(C)} \MSS\ (KCI) edge orientations mostly match the real data Sachs network~\citep{sachs2005causal}. Non-matching edges (red) are posited to be involved in cycles and of ambiguous orientation in the literature.}
    \label{fig:unified_empirical}
\end{figure}

\textbf{Oracle \MSS\ rates match the theory.} Having theory on learning rate bounds under sparsity, we now seek to understand how the empirical performance of \MSS\ and pooled PC depends on a variety of graph and sparsity settings, both which our theory does and does not address. We consider random DAGs over $6$ variables with edge density $0.3$. Five environments are sampled, in each of which half of the mechanisms shift. \looseness-1 In~\cref{fig:sim_oracle_rates}, we hold all of these settings fixed and vary one at a time across $50$ repetitions, comparing the recall of the two methods.\footnote{Oracle methods used code from the \textit{causaldag} package [3-Clause BSD license].} Precision is always perfect under the oracle test.

In the first two plots, the empirical results match what the theory predicts. First, per~\cref{cor:identifiability}, with more environments \MSS\ learns the entire graph while pooled PC learns nothing but the original MEC. Second, per~\cref{thm:consi_dags}, the learning rate decreases when shifts are either uncommon or frequent. Furthermore, we see that differences between the two approaches are accentuated in sparse and dense DAGs, as well as in smaller DAGs. Note that sparsity is a fixed fraction of the variables, and hence larger DAGs experience more shifts in absolute terms. See~\cref{appendix:supporting_experiments} for additional oracle experiments.

\textbf{MSS performance depends on the chosen estimator.} 
\looseness-1 Next, we study how the choice of estimator and type of score affects performance. We use two popular conditional independence tests, the Fisher-Z partial correlation test~\citep{kalisch2007estimating} and the Kernel Conditional Independence (KCI) test~\citep{zhang2011kernel}, as well as the invariant residual test using a generalized additive model (GAM), a top-performing ICP method~\citep{heinze2018invariant}.\footnote{KCI and Fisher-Z are implemented by the \textit{causal-learn} package [GNU General Public License].} Each detects a change if the test \textit{p}-value is less than $\alpha := 0.05 / d$, a Bonferroni correction to bound the false positive rate for each scored DAG. Although the theory pertains to counting shifts, in settings where a ``hard" hypothesis test has low power it may be of more value to use a ``softer" score (e.g., based on \textit{p}-values) to quantify changes along a continuous spectrum. We propose the following ``soft'' score (see~\cref{appendix:details} for further details):
\begin{equation*}
    \textstyle
    \widehat{\MSS}_j(G; \Dcal) = \sum_{e=1, e'>e}^{n_{\Ecal}} \left[1 - \text{\textit{p}-value}\left(\PP^e(X_j|\PA_j^G) \neq \PP^{e'}(X_j|\PA_j^G)\right)\right].
\end{equation*}
DAGs are generated on six variables with edge density $0.3$. Three environments are sampled, with $500$ samples and two mechanism shifts per environment. We vary each variable while holding the others fixed and compare results across $50$ repetitions.

\looseness-1 As seen in~\cref{fig:unified_empirical}(A), the Fisher-Z test performs noticeably worse, presumably due to unmet parametric assumptions, while the nonparametric approaches do well, particularly at higher sample sizes. As with the oracle, the true DAG is recovered with enough environments. The ``soft'' versions achieve higher recall, since there is a unique minima, at the cost of worse precision (see~\cref{appendix:soft_score} for greater discussion). The ``soft" KCI seems to be best, suggesting that KCI models the data the best. In practice, it is crucial for a method to model the data well for the \textit{p}-value to be valid~\citep{shah2018hardness}.

\textbf{MSS compares favorably against other methods.} \looseness-1 
Next, we compare \MSS\ to relevant existing methods: the \textit{Minimal changes (MC)} approach~\citep{ghassami2018multi} tests pairs of environments for changes in the parameters of a linear model, and \citet{huang2020causal} provide a nonparametric version---a two-stage approach which first uses PC with the KCI test on the pooled data. We investigate if pooling data loses information under empirical tests, and how the nonparametric pairwise \MSS\ test compares. 

\looseness-1 In~\cref{fig:unified_empirical}(B), we compare these approaches in the previously studied experimental setting. Pooled PC has quite high precision (\rebuttal{as it pools all of the data}), but suffers from lower recall, especially with more environments when recall is no better than the base MEC. In contrast, ``hard'' KCI has higher recall \rebuttal{at the cost of some precision when individual environments have few samples}. The parametric MC works surprisingly well, yet slightly worse than KCI. Overall, \MSS\ seems to combine the best of both approaches: the value of pairwise comparisons from MC with the flexibility of incorporating various nonparametric estimators. Additional experiments in~\cref{appendix:supporting_experiments} confirm this in the bivariate case.

\textbf{Protein network discovery.}
\looseness-1
As an illustration of our method in practice, we conduct a case-study application of \mbox{\MSS\,} for causal discovery on a well-studied cytometry dataset~\mbox{\citep{sachs2005causal}} consisting of $9$ experimental environments of $11$ cellular proteins. Starting from the Sachs MEC, we apply the \mbox{\MSS\,} using the KCI test, which appears to perform the best among plug-in estimators for \mbox{\MSS\,} in our simulations.

\looseness-1 The DAG which minimizes the \MSS\, is the unique minimizer and is visualized in~\cref{fig:unified_empirical}(C). Learned edge orientations mostly match the Sachs network~\citep{sachs2005causal}. Non-matching edges, shown in red, are posited to be involved in cycles and of ambiguous orientation in the literature: PIP2 $\to$ PIP3 is known to be a cyclical relation~\citep{sachs2005causal, ramsey2018fask}, Mek $\to$ Raf is indeed found by many other methods~\citep{eaton2007exact, ramsey2018fask, mooij2020joint}, and although there is not a detailed discussion of the PKA, PKC, P38 triangle, there is ambiguity in the edge directions among approaches~\citep{mooij2020joint, ramsey2018fask, eaton2007exact}, as we discuss in more detail in~\cref{appendix:cytometry}.

\section{Discussion}\label{sec:discussion}
\textbf{Sparse shifts as a relaxation of the i.i.d.\ assumption.}
Distribution shifts are a common violation of the i.i.d.\ assumption and a problematic source of error in practice.
It has been argued that the issue of robustness to natural shifts is connected to causality~\citep{scholkopf2012on,Bareinboim_NIPS2014,zhang_multi-source_2015,peters2016causal, veitch2021counterfactual, arjovsky2020invariant, Kugelen2021selfsupervised}.
\looseness-1 We have shown that viewing a shift in distribution through the causal factorization permits a useful relaxation of the i.i.d.\ assumption, facilitating causal discovery: if there are no shifts (i.i.d.), we can only identify equivalence classes. If shifts are unrestricted, we cannot meaningfully transfer across distributions~\citep{scholkopf2012on}, and ideas such as Simon's invariance criterion~\citep{Hoover06} and ICP~\citep{peters2016causal} do not help. However, if shifts occur {\em sparsely}---as we formalize---we can provably use this as a learning signal to infer the causal structure.

\textbf{The Mechanism Shift Score (\MSS) and prior methods.}
\looseness-1 The \MSS\ framework extends previous causal discovery work limited to linear mechanisms~\citep{ghassami2018multi}, just as nonlinear ICP~\citep{heinze2018invariant} extended the initially limited ICP approach~\citep{peters2016causal} beyond linear models. Additionally, we have provided a graph-theoretic analysis proving why pairwise comparisons are actually useful; the learning rates we have established apply to these previous works and provide insight on the role of sparsity. Although \citet{huang2020causal} provide a thorough nonparametric approach and analysis, we have demonstrated that the first stage of their two-stage-approach is not suited for learning beyond the MEC since it pools all the data. The \MSS\ is both nonparametric and explicitly leverages sparsity through pairwise comparisons.

\textbf{Beyond causal discovery.}
\looseness-1 Once the causal graph is known, conditional distributions and hence an entire causal graphical model can be learned. This is harder than learning a statistical model, but has various advantages, especially when distributions shift, as they do in reality.
E.g., we may be able to use such a model for causal reasoning, i.e., estimating a certain causal effect.
We also expect that \MSS\ can serve as a useful inductive bias for causal representation learning, similar to how invariant prediction~\citep{scholkopf2012on,peters2016causal} inspired invariant risk minimization~\citep{arjovsky2020invariant}; recent work has started to explore this~\citep{lachapelle2022disentanglement,lippe2022citris}.

\textbf{Empirical performance \rebuttal{and limitations.}}
\looseness-1 We infer causal structure through a flexible score-based method; as empirically demonstrated, strong results can be obtained by multiple estimators when the assumption of sparsity is met. \rebuttal{This requires hypothesis tests which can accurately obtain the full MEC as a starting point and subsequently test for distribution changes. Among additional assumptions including the stringent pseudo-causal sufficiency, sparsity is necessary and yet not easily verifiable; SMS is indeed a  \textit{hypothesis} regarding causal systems, not a fact.} We conjecture that under an oracle, the ``hard'' \MSS\ is equivalent to the PC algorithm pooled pairwise across environments. It is worth noting that the ``hard'' approach may still be useful under dense changes if only a sparse number of them are large enough to be \textit{empirically discernible}. Thus the empirical method can actually outperform the oracle baseline and be useful even if the assumption of sparsity is unmet.
We also assume that a partition into environments is known, but environment inference techniques may help relax this~\citep{creager2021environment}.

\textbf{Outlook and conclusion.} \looseness-1 Imagining causal models on an axis of complexity, from the microscopic physical laws of nature to a simplified set of variables and relationships, we transition from a system with no mechanism shifts (effectively a dynamical system) to a system in which all mechanisms shift (due to many unmeasured causes). In the middle, we posit only a sparse number of shifts to be empirically discernible. While {\em all (causal) models are wrong}, the one which is most invariant to shifts may be the best candidate for supporting robust and transferable inference.

\begin{ack}
\looseness-1
We thank Jonas K\"ubler, Junhyung Park, Krikamol Muandet, Luigi Gresele, the T\"ubingen causality team, and the anonymous reviewers for helpful comments.
This work was supported by the German Federal Ministry of Education and Research (BMBF): Tübingen AI Center, FKZ: 01IS18039A, 01IS18039B; and by the Machine Learning Cluster of Excellence, EXC number 2064/1 – Project number 390727645. Ronan Perry was supported by a Fulbright Germany research fellowship.
\end{ack}

\bibliographystyle{plainnat}
\bibliography{references}

\newpage
\renewcommand{\baselinestretch}{1}
\appendix
\addcontentsline{toc}{section}{Appendices}%
\part{Appendices} %
\changelinkcolor{black}{}
\parttoc%
\changelinkcolor{red}{}

\section{Graph terminology}\label{appendix:graph_theory}

A directed graph $G = (V, T)$ is an object consisting of a set of vertices $V$ and a set of ordered pairs of vertices $T\subset V\times V$ corresponding to directed edges in $G$. A \textit{path} is a sequence of vertices $(V_{i_1},\dots,V_{i_n})$ with $n\geq 2$ such that $V_{i_k} \to V_{i_{k+1}}$ or $V_{i_{k}} \leftarrow V_{i_{k+1}}$ and in a \textit{directed path} $V_{i_{k}} \to V_{i_{k+1}}$ for all $k$. In graph $G$: the \textit{children} $\CH_j^G$ of $V_j$ are all $V_{m}$ such that $V_{j} \to V_m$, the \textit{parents} $\PA_j^G$ of $V_j$ are all $V_{m}$ such that $V_{m} \to V_j$, the \textit{ancestors} $\AN_j^G$ of $V_j$ are all $V_m$ such that there exists a directed path $(V_m,\dots,V_j)$, and the \textit{descendants} $\DE_j^G$ of $V_j$ are $V_j$ and all $V_m$ such that there exists a directed path $(V_j,\dots,V_m)$. The graph superscript will be omitted unless needed.  A \textit{cycle} is a path such that $V_{i_1} = V_{i_n}$ and a \textit{directed acyclic graph (DAG)} is a directed graph with no directed cycles.

On a path $(V_{i_1},\dots,V_{i_k},\dots,V_{i_n})$, we say variable $V_{i_k}$ is a \textit{collider} if $V_{i_{k-1}} \to V_{i_k}$ and $V_{i_k} \leftarrow V_{i_{k+1}}$. A subset $\Zb \in \Vb \setminus \{V_{i_1}, V_{i_n}\}$ \textit{blocks} the path if either (i) $\Zb$ contains
at least one non-collider vertex on the path or (ii) the path contains a collider with no descendants in $\Zb$ (this includes the collider itself by the descendant definition). With this terminology, we say that on the disjoint variable sets $\Ab$, $\Bb$, and $\Zb$, $\Ab$ is \textit{d-separated} from $\Bb$ by $\Zb$ iff every path between $\Ab$ and $\Bb$ is blocked by $\Zb$ ~\cite[Def.~6.1]{peters2017elements}. This is denoted as $\Ab \independent_{\!\!G}~\Bb~|~\Zb$. If $\Ab$ and $\Bb$ are not d-separated, and hence there exists an unblocked path, we say that they are \textit{d-connected}.

\clearpage
\section{Full proofs}\label{appendix:full_proofs}

\rebuttal{Our first result is not focused on in the main text but is nonetheless interesting in relation to related works and creating a coherent multi-environment causal discovery framework. Specifically, it shows that the \mbox{\MSS\,} solution set is an equivalence class of DAGs. Our main results demonstrate when, how, and under what conditions this equivalence class shrinks to only the true DAG.}
\begin{proposition}
\rebuttal{$\Gcal_{\MEC}^{\min}$ is a $\Psi$-MEC with respect to the true but unknown interventional targets $\{\Ical^e\}_{e=1}^{n_e}$.}
\end{proposition}
\begin{proof}
First we must introduce the $\Psi$-MEC concept as introduced by~\citet{jaber_causal_2020}. Recall that $\Ical^e$ denotes the set of mechanisms which change in environment $e \in |n_e| := \{1,\dots,n_e\}$ as compared to the unknown baseline environment. The following definition for an augmented graph without latent variables comes from \citet{jaber_causal_2020}.
\begin{definition}[Augmented graph~\citep{jaber_causal_2020}]
Consider a DAG $G = (\Xb, \Tb)$. Let $$|n_e|^2 := \{(e, e') | e \in |n_e|, e' \in |n_e|, e > e'\}$$ denote all unordered pairs of environments. The augmented DAG
\begin{equation*}
    Aug_\Ical(G) := (\Xb \cup \{E^{e, e'}\ | (e, e') \in |n_e|^2\}, \Tb \cup \big(\cup_{(e, e') \in |n_e|^2}\{E^{e, e'} \to X_i | i \in \Ical^e \triangle \Ical^{e'}\}\big))
\end{equation*} incorporates environmental variable nodes $E^{e, e'}$ for all pairwise environments and edges to a variables $X_j$ if its mechanism differs between environments $e$ and $e'$. $\triangle$ denotes the symmetric set difference.
\end{definition}

From this definition, \citet{jaber_causal_2020} provide the following results

\begin{corollary}[\citet{jaber_causal_2020}]\label{corollary:jaber}
Given two DAGs $G_1$ and $G_2$ on the same vertices without latent variables, and sets of interventional targets $\Ical_1^e$ and $\Ical_2^e$ for $e \in |n_e|$, the pairs $(G_1, \{\Ical_1^e\}_{e \in |n_e|}$ and $(G_2, \{\Ical_2^e\}_{e \in |n_e|}$ are $\Psi$-Markov equivalent iff $Aug_{\Ical_1}(G_1)$ and $Aug_{\Ical_2}(G_2)$ have (1) the same skeleton and (2) the same v-structures. That is, iff the augmented graphs are in the same Markov equivalence class. 
\end{corollary}

Now recall that by definition
\begin{equation*}
    \Gcal_{\MEC}^{\min} := \argmin_{G \in \Gcal_{\MEC}} \MSS(G; \Pcal)
\end{equation*}

($\Rightarrow$) To show that all elements of $\Gcal_{\MEC}^{\min}$ are $\Psi$-Markov equivalent with respect to the true targets $\{\Ical^e\}_{e \in |n_e|}$, we need to verify the two corollary conditions.

\begin{itemize}
    \item \textbf{Skeleton}: Since $\Gcal_{\MEC}^{\min} \subseteq \Gcal_{\MEC}$, by definition of the \MEC\ all unaugmented DAGs in the set share the same skeleton. Since the augmented edges are defined solely with respect to the intervention targets, all augmented DAGs also share the same skeleton.
    \item \textbf{v-structures}: Since $\Gcal_{\MEC}^{\min} \subseteq \Gcal_{\MEC}$, by definition of the \MEC\ all unaugmented DAGs in the set share the same v-structures. Augmented DAGs form v-structures through edges from the augmented variables $E^{e,e'}$. Let $E^{e,e'} \to X_i$ be an augmented edge in $Aug_\Ical(G^*)$ with a v-structure formed by $X_j \to X_i$; thus there is no edge $E^{e,e'} \to X_j$ and the mechanism of $X_j$ is truly invariant across environments $e$ and $e'$. If a graph $G$ does not contain this v-structure, then it must orient the edge $X_i \to X_j$. Thus the mechanism $\PP^e_\Xb(X_j~|~\PA_j^{G}) \neq \PP^{e'}_\Xb(X_j~|~\PA_j^{G})$ differs due to the unblocked path from $E$ to $X_j$, $\MSS(G) > \MSS(G^*)$, and thus $G \not\in \Gcal_{\MEC}^{\min}$.
\end{itemize}
Thus, all augmented graphs in $\Gcal_{\MEC}^{\min}$ contain the same skeleton and v-structures.

($\Leftarrow$) To show that $\Gcal_{\MEC}^{\min}$ is exactly the $\Psi$-MEC, we need to show that there cannot exist a DAG $G \not\in \Gcal_{\MEC}^{\min}$ but which satisfies the $\Psi$ equivalence conditions with $G^*$.

First, any DAG $G$ not in the MEC must differ in skeleton or v-structures and thus also not be in the $\Psi$-MEC. So, assume $G$ is an element of the MEC but $G \not\in \Gcal_{\MEC}^{\min}$ and thus $\MSS(G) > \MSS(G^*)$. By~\cref{cor:edge_flip}, there exists some truly invariant $X_j$ with either \textbf{case (ii):} an unblocked path to an unconditioned parent  or \textbf{case (iii):} an unblocked path to a conditioned child. In both cases, we show that a v-structure differs between $G$ and $G^*$ which completes our proof by contradiction.

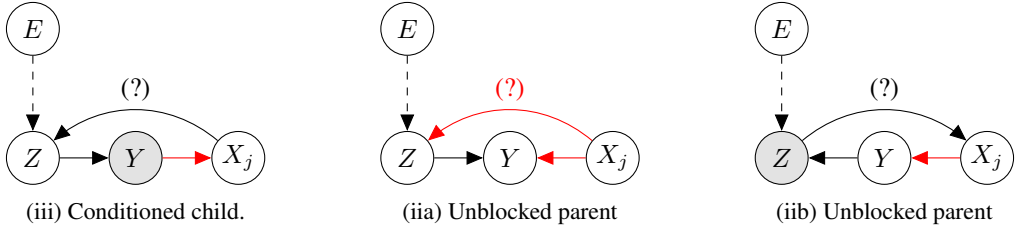
\begin{figure}[t]
    \centering
    \captionsetup[subfigure]{labelformat=empty}
    \newcommand{\xshift}{1em}
    \newcommand{\yshift}{1em} 
    \begin{subfigure}[b]{0.25\textwidth}
        \centering
        \begin{tikzpicture}
            \centering
            \node (X) [latent] {$X_j$};
            \node (Y) [latent, left=of X, xshift=\xshift, yshift=0, inner sep=0em, fill={rgb:black,1;white,8}]{$Y$};
            \node (Z) [latent, left=of Y, xshift=\xshift, yshift=0, inner sep=0em]{$Z$};
            \node (E) [latent, above=of Z]
            {$E$};
            \edge[red]{Y}{X};
            \edge[]{Z}{Y};
            \edge[dashed]{E}{Z};
            \path[bend left=40, style=<-] (Z) edge node [above]{(?)} (X);
        \end{tikzpicture}
        \caption{(iii) Conditioned child.}
    \end{subfigure}%
    \hspace{0.1\textwidth}
    \begin{subfigure}[b]{0.25\textwidth}
        \centering
        \begin{tikzpicture}
            \centering
            \node (X) [latent] {$X_j$};
            \node (Y) [latent, left=of X, xshift=\xshift, yshift=0, inner sep=0em]{$Y$};
            \node (Z) [latent, left=of Y, xshift=\xshift, yshift=0, inner sep=0em]{$Z$};
            \node (E) [latent, above=of Z]
            {$E$};
            \edge[red]{X}{Y};
            \edge[]{Z}{Y};
            \edge[dashed]{E}{Z};
            \path[bend left=40, style=<-, red] (Z) edge node [above]{(?)} (X);
        \end{tikzpicture}
        \caption{(iia) Unblocked parent}
    \end{subfigure}%
    \hspace{0.1\textwidth}
    \begin{subfigure}[b]{0.25\textwidth}
        \centering
        \begin{tikzpicture}
            \centering
            \node (X) [latent] {$X_j$};
            \node (Y) [latent, left=of X, xshift=\xshift, yshift=0, inner sep=0em]{$Y$};
            \node (Z) [latent, left=of Y, xshift=\xshift, yshift=0, inner sep=0em, fill={rgb:black,1;white,8}]{$Z$};
            \node (E) [latent, above=of Z]
            {$E$};
            \edge[red]{X}{Y};
            \edge[]{Y}{Z};
            \edge[dashed]{E}{Z};
            \path[bend left=40, style=->] (Z) edge node [above]{(?)} (X);
        \end{tikzpicture}
        \caption{(iib) Unblocked parent}
    \end{subfigure}%
    \caption{\small Cases in the proof of the $\Psi$-MEC. The candidate DAG $G$ is drawn, with edges differing from the true DAG in red and question marks above edges which may or may not exist. Shaded nodes are conditioned on, as implied in each case. The dashed line represents a d-connecting path. Cycles are naturally contradictions with the assumption of acyclicity and }
    \label{fig:psi_mec_proof}
\end{figure}

For two environments, let $X_j$ have an invariant mechanism in $G^*$, $Y$ adjacent to $X_j$, and $E$ be the environmental variable possible equal to $Z$.

\begin{itemize}
    \item \textbf{(iii) Conditioned child:} As seen in the first plot of~\cref{fig:psi_mec_proof}, the graph $G$ incorrectly contains the edge $Y \to X_j$ which leads us to condition on the child $Y$ and thus open a path from $E$ to $X_j$. If there is no edge between $Z$ and $X_j$, then the $Z - Y - X_j$ v-structure differs between $G$ and $G^*$. If there is an edge, it must orient $X_j \to Z$ or else block the path through $Y$, but then $G$ contains the cycle $Z \to Y \to X_j \to Z$ and thus the edge cannot exist at all. If $Z = E$, then of course $X_j \not\to E$ and $E \to X_j$ is not permitted under invariance of $X_j$.
    
    \item \textbf{(ii) Unconditioned parent:}  As seen in the second two plots of~\cref{fig:psi_mec_proof}, $G$ experiences a change in $X_j$ due to an unblocked path through the true parent $Y$; therefore $G$ contains $X_j \to Y$ and the path from $(E,\dots,X_j)$ contains some $Z - Y$ edge. There are two orientations of the edge $Z - Y$ to consider \textbf{Subcase (a) $Z \to Y$:} If the there is no edge between $Z$ and $X_j$ (e.g. if $Z = E$), then clearly a v-structure differs. If there is an edge, to keep $Z$ unblocked $G$ must contain $X_j \to Z$. But $X_j \to Z$ must be incorrectly directed since otherwise would imply the cycle $Z \to Y \to X_j \to Z$ in the true DAG $G^*$. Thus we have the unblocked parent $Z$ and so this brings us back to the start of case (ii); with a finite number of variables for the \textit{parent} $Z$ to be (since repetition would create a cycle) we will eventually end up in either subcase (b) or with $Z = E$ in which case the edge $X_j \to E$ cannot exist and so the v-structure $E \to Y \leftarrow X_j$ differs. \textbf{Subcase (b) $Z \leftarrow Y$:} In order to open the path through $Y$ to $X_j$, the collider $Z$ must be conditioned on in $G$ which implies $Z \to X_j$; this forms a cycle in $G$ and thus a contradiction of $G$ being a DAG.
\end{itemize}
\end{proof}

\subsection{Proof of Lemma~\ref{lem:edge_flip}}

\lemFlip*

\begin{proof}
($\Rightarrow$)
If $\PP(X_j~|~\Zb)$ changes across environments $E$, then $X_j \not\independent E~|~\Zb$. The global Markov property of the CGM states that d-separation implies conditional independence, and thus by the contra-positive the d-connectedness relationship follows. 

($\Leftarrow$)
d-connectedness implies conditional dependence by faithfulness, and thus a change across environments.
\end{proof}

\subsection{Proof of Corollary~\ref{cor:edge_flip}}

\corFlip*

\begin{proof}
By Lemma 4.3, $\PP(X_j~|~\Zb)$ changes iff $X_j \not\independent_{G_{\Xb \cup E}} E~|~\Zb$ (d-connection) and equivalently, iff there is an unblocked path $(E,\dots,X_j)$ in $G_{\Xb \cup E}$. We assume a generic path and through casework establish that it is unblocked if and only if one of cases (i), (ii), or (iii) holds. The casework is visualized in Figure 7.

Either $(E,\dots,X_j)$ is just $E \to X_j$ [case (i)] or there must exist $W$ such that $(E,\dots,W,X_j)$ and $W$ is either (a) a collider or (b) not a collider.

(a) If $W$ is a collider, it necessarily a child of $X_j$ The collided path in unblocked iff $W \not\independent_{G_{\Xb \cup E}} E~|~\Zb\setminus W$ and some descendant $W' \in \DE_W^G$ of $W$ is conditioned on. Thus $W' \in \Zb \subset \PA_j^G \cup \CH_j^G$. Without loss of generality, assume $W'$ is the closest descendant to $W$ and hence the path $(W,\dots,W')$ is unblocked by $\Zb$. $W'$ cannot be a parent of $X_j$, else induce the cycle $(X_j, W, \dots, W', X_j)$, and so must be a child and case (iii) holds. Specifically, there exists $W'$ such that $W' \not\independent_{G_{\Xb \cup E}} E~|~\Zb\setminus W'$ and $W'$ is a child in $\Zb$.

(b) If $W$ is not a collider, by definition the path is unblocked iff $W \not\independent_{G_{\Xb \cup E}} E~|~\Zb$ and $W \not \in \Zb$. If $W$ is a parent of $X_j$ (since they are adjacent), case (ii) holds. If $W$ is a child of $X_j$, because $E$ has an outgoing edge there must exist some collider $C$ on the path such that $(E,\dots,C,\dots,W,X_j)$ and the subpath from $W$ to $C$ is directed into $C$. The condition $W \not\independent_{G_{\Xb \cup E}} E~|~\Zb$ holds iff some descendant $W'$ of $C$ is in $\Zb$. As before, $W'$ cannot be a parent of $X_j$ or else induce a cycle, and so it must be a child and case (iii) holds.

\begin{figure}[t]
    \centering
    \captionsetup[subfigure]{labelformat=empty}
    \newcommand{\xshift}{1em}
    \newcommand{\yshift}{1em} 
    \begin{subfigure}[b]{0.24\textwidth}
        \centering
        \begin{tikzpicture}
            \centering
            \node (X) [latent] {$X_j$};
            \node (W) [latent, left=of X, xshift=\xshift, yshift=0, inner sep=0em] 
            {$W$};
            \node (E) [latent, left=of W, xshift=\xshift]
            {$E$};
            \edge[]{W}{X};
            \edge[dashed]{E}{W};
        \end{tikzpicture}
        \caption{(a) $W$ is a parent.}
    \end{subfigure}%
    \hspace{0.02\textwidth}
    \begin{subfigure}[b]{0.35\textwidth}
        \centering
        \begin{tikzpicture}
            \centering
            \node (X) [latent] {$X_j$};
            \node (W) [latent, left=of X, xshift=\xshift, yshift=0, inner sep=0em] 
            {$W$};
            \node (W2) [latent, below=of W, xshift=0, yshift=\yshift, fill={rgb:black,1;white,8}, inner sep=0em]
            {$W'$};
            \node (E) [latent, left=of W, xshift=\xshift]
            {$E$};
            \edge[]{X}{W};
            \edge[dashed]{W}{W2};
            \edge[dashed]{E}{W};
            \path[bend right=30, style=<-] (W2) edge node {} (X);
        \end{tikzpicture}
        \caption{(b) $W$ is a child and collider.}
    \end{subfigure}%
    \hspace{0.02\textwidth}
    \begin{subfigure}[b]{0.35\textwidth}
        \centering
        \begin{tikzpicture}
            \centering
            \node (X) [latent] {$X_j$};
            \node (W) [latent, left=of X, xshift=\xshift, yshift=0, inner sep=0em] 
            {$W$};
            \node (W2) [latent, left=of W, xshift=\xshift, yshift=0]
            {$C$};
            \node (E) [latent, left=of W2, xshift=\xshift]
            {$E$};
            \node (V) [latent, below=of W2, yshift=\yshift, fill={rgb:black,1;white,8}]
            {$W'$};
            \edge[]{X}{W};
            \edge[dashed]{W}{W2};
            \edge[dashed]{E}{W2};
            \path[bend right=30, style=<-] (V) edge node {} (X);
            \edge[dashed]{W2}{V};
        \end{tikzpicture}
        \caption{(b) $W$ is a child but not a collider.}
    \end{subfigure}%
    \caption{\small Cases from the proof of~\cref{cor:edge_flip}. Case (a) iff case (ii). Case (b) induces two subcases either of which occur iff case (iii).}
    \label{fig:edge_flip_proof}
\end{figure}
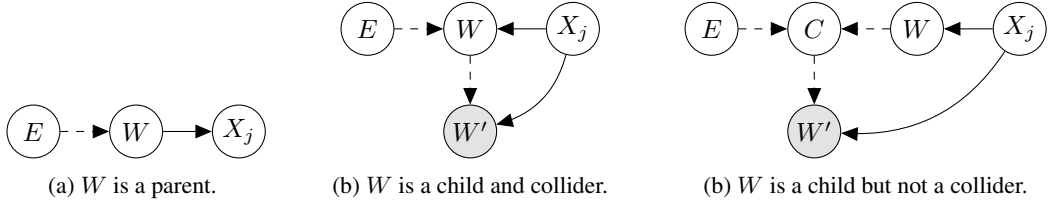

\end{proof}

\subsection{Proof of Lemma~\ref{lem:consis_mech}}

\lemConsistency*

\begin{proof}
By Assumption~\ref{assumption:shared_mechs}, the distribution $\PP^e_\Xb$ in each environment $e \in \{1,\dots, n_{\Ecal}\}$ is the result of changing mechanisms from some underlying yet unknown distribution $\PP_\Xb$. Let $\Delta^{e,e'}(X_j)$ denote the event $\II[\PP^e_\Xb(X_j~|~\PA_j^{G^*}) \neq \PP^{e'}_\Xb(X_j~|~\PA_j^{G^*})]$ that the mechanism of variable $X_j$, with respect to the true graph $G^*$, changes across environments $e$ and $e'$. Abbreviate $\rho^{e,e'}_j := \Pr[\Delta^{e,e'}(X_j) = 1]$.

Since $\PA_j^{G^*} \neq \PA_j^G$ and $G$ shares the same skeleton as $G^*$, at least one edge must be oriented incorrectly in $G$. In the conditioning set $\PA_j^G$ according to the incorrect graph $G$, there thus exists either an unconditioned true parent $Z \in \PA_j^{G^*} \setminus \PA_j^G$ or a conditioned-upon true child $Z \in \CH_j^{G^*} \cap \PA_j^G$. By \cref{cor:edge_flip}, we know that if $Z$ is not d-separated from $E$ in the augmented graph, then the conditional $\PP(X_j~|~\PA_j^G)$ changes across $E$. This occurs at least if the mechanism of $Z$ directly changes, e.g. there is the edge $E \to Z$ in the augmented graph.

Consider first the case of two environments. We know from \cref{prop:min_shift} that $\MSS_j(G^*; \Pcal)$ cannot be greater than $\MSS_j(G; \Pcal)$, and will be less if the mechanism of $X_j$ remains invariant while the mechanism of $Z$ changes. By the assumption of independent changing mechanisms,
\begin{align*}
    &\Pr[\MSS_j(G^*; \{\Dcal^1, \Dcal^2\}) = \MSS_j(G; \{\Dcal^1, \Dcal^2\})]\\
    &= 1 - \Pr[\MSS_j(G^*; \{\Dcal^1, \Dcal^2\}) < \MSS_j(G; \{\Dcal^1, \Dcal^2\})]\\
    &\leq 1 - \Pr[\Delta^{1,2}(X_j) = 0, \Delta^{1,2}(Z) = 1]\\
    &= 1 - \Pr[\Delta^{1,2}(X_j) = 0]\Pr[\Delta^{1,2}(Z) = 1]\\
    &= 1 - (1 - \rho^{1,2}_j) \rho^{1,2}_Z
\end{align*}

Given $n_{\Ecal} > 2$ environments, it follows that
\begin{align*}
    & \,\,\, \Pr[\MSS_j(G^*; \Pcal) = \MSS_j(G; \Pcal)]\\
    &= \Pr\left[\bigcap_{e, e'>e} \MSS_j(G^*, \{\Dcal^e, \Dcal^{e'}\}) = \MSS_j(G, \{\Dcal^e, \Dcal^{e'}\})\right] \\
    &\leq \Pr\left[\bigcap_{e \in \{1,\dots,\lfloor\Ecal/2\rfloor\}} \MSS_j(G^*, \{\Dcal^{2e-1}, \Dcal^{2e}\}) = \MSS_j(G, \{\Dcal^{2e-1}, \Dcal^{2e}\})\right]\\
    &= \prod_{e \in \{1,\dots,\lfloor\Ecal/2\rfloor\}} \Pr\left[ \MSS_j(G^*, \{\Dcal^{2e-1}, \Dcal^{2e}\}) = \MSS_j(G, \{\Dcal^{2e-1}, \Dcal^{2e}\})\right]\\
    &\leq \prod_{e \in \{1,\dots,\lfloor\Ecal/2\rfloor\}} \left(1 - (1 - \rho^{2e-1,2e}_j) \rho^{2e-1,2e}_Z\right)\\
    &\leq \left(1 - \min_{e \in \{1,\dots,\lfloor\Ecal/2\rfloor\}}(1 - \rho^{2e-1,2e}_j) \rho^{2e-1,2e}_Z\right)^{\lfloor n_{\Ecal} / 2\rfloor}.
\end{align*}
Since $Z$ is arbitrary, we construct an upper bound using the worst case, in which a variable frequently or rarely changes.
\begin{align*}
    1 - \min_{e \in \{1,\dots,\lfloor\Ecal/2\rfloor\}}(1 - \rho^{2e-1,2e}_j) \rho^{2e-1,2e}_Z
    &\leq 1 - (1 - \max_{e \in \{1,\dots,\lfloor\Ecal/2\rfloor\}} \rho^{2e-1,2e}_j) \min_{e \in \{1,\dots,\lfloor\Ecal/2\rfloor\}} \rho^{2e-1,2e}_Z\\
    &\leq 1 - (1 - \max_{e, e' \neq e} \rho^{e,e'}_j) \min_{e, e' \neq e} \rho^{e,e'}_Z\\
\end{align*}
and so to acquire the final bound with simplified notation, for any variable $X_i$ denote the minima and maxima of $\rho^{e,e'}_i$ across any two environments with $\rho_i^{\lb}$ and $\rho_i^{\ub}$, respectively.
\end{proof}

\subsection{Proof of Theorem~\ref{thm:consi_dags}}
\thmConsistency*

\begin{proof}
Since $\Pr[\Gcal_{\MEC}^{\min} = \{G^*\}] = 1 - \Pr[\Gcal_{\MEC}^{\min} \neq \{G^*\}]$ and by~\cref{lem:consis_mech},
\begin{align*}
    \Pr[\Gcal_{\MEC}^{\min} \neq \{G^*\}]
    &= \Pr\left[\bigcup_{G \in \Gcal_{\MEC} \setminus \{G^*\}} \MSS(G^*; \Pcal) = \MSS(G; \Pcal) \right]\\
    &\leq \sum_{G \in \Gcal_{\MEC}} \Pr\left[\MSS(G^*; \Pcal) = \MSS(G, \Dcal)\right]\\
    &\leq \sum_{G \in \Gcal_{\MEC}} \Pr\left[\sum_j \MSS_j(G^*; \Pcal) = \sum_j \MSS_j(G, \Dcal)\right]\\
    &= \sum_{G \in \Gcal_{\MEC}} \Pr\left[\sum_j \MSS_j(G^*; \Pcal) = \sum_j \MSS_j(G, \Dcal)\right]\\
    & = \sum_{G \in \Gcal_{\MEC}} \Pr\left[\bigcap_j \MSS_j(G^*; \Pcal) = \MSS_j(G, \Dcal)\right]\\
    & \leq \sum_{G \in \Gcal_{\MEC}} \min_j \Pr\left[\MSS_j(G^*; \Pcal) = \MSS_j(G, \Dcal)\right]\\
    &\leq \sum_{G \in \Gcal_{\MEC}} \min_j \left(1 - (1 - \rho_j^{\ub}) \min_{i} \rho_i^{\lb}\right)^{\lfloor n_{\Ecal} / 2\rfloor}\\
    &\leq \sum_{G \in \Gcal_{\MEC}} \left(1 - (1 - \min_j \rho_j^{\ub}) \min_{i} \rho_i^{\lb}\right)^{\lfloor n_{\Ecal} / 2\rfloor}\\
    &= |\Gcal_{\MEC}| \left(1 - (1 - \min_j \rho_j^{\ub}) \min_{i} \rho_i^{\lb}\right)^{\lfloor n_{\Ecal} / 2\rfloor}.
\end{align*}
\end{proof}

\clearpage
\section{Details of assumptions and methods}\label{appendix:details}

\subsection{Pseudo causal sufficiency and the Independent Causal Mechanisms (ICM) assumption}

\citet{huang2020causal} introduced the idea of psuedo-causal sufficiency (\cref{assumption:pseudo_suff}) and provide a useful discussion on its relation to results on soft interventions by \citet{eberhardt2007interventions}. \citet{guo2022causal} provide a useful formalization of multi-environment data, specifically through a
plate-notation representation.
An environment $e$ specifies parameters of the causal mechanisms in the CGM over $\Xb$; we can think of environments as encapsulating specific experimental settings, or broad contexts such as climate or time~\citep{mooij2020joint}. 
Under the context of $e$, there is some distribution $\PP^e_\Xb$ and we observe a dataset sampled i.i.d.
The ICM assumption tells us that the parameters for each causal mechanism in an environment are chosen or sampled independently, and thus in the augmented CGM the edges from $E$ appear independently.

Within each environment, i.e., when we condition on $E$, the environmental parameters are fixed; thus we are in the typical i.i.d.\ setting and causal sufficiency is implied by the CGM. However, without conditioning on $E$, the environmental parameters are not fixed and across two samples either all remain the same (if the samples are in the same environment) or some change. This dependence between samples through the parameters defined by $E$ is the result of $E$ being a confounder; thus causal sufficiency cannot hold over $\Xb$ without conditioning on $E$. Because $E$ is not necessarily a true causal variable but rather an environment encoding a fixed set of unmeasured variables, \citet{huang2020causal} call it a \textit{pseudo-confounder}. It is worth noting that the second stage of the approach of \citet{huang2020causal} relies on a novel kernel-based test, which computes a measure of mechanism dependence across all samples. They correctly compare the test statistics rather than examine \textit{p}-values because the dependence between samples without controlling for environment would lead to a samll \textit{p}-value even if the mechanisms were independent.

\subsection{The \textit{p}-values ``soft" score}\label{appendix:soft_score}

We provide further details on the \textit{p}-value ``soft" score. Recall the modified score definition to be
\begin{equation*}
    \textstyle
    \widehat{\MSS}_j(G; \Dcal) = \sum_{e=1, e'>e}^{n_{\Ecal}} \left[1 - \text{\textit{p}-value}\left(\PP^e(X_j|\PA_j^G) \neq \PP^{e'}(X_j|\PA_j^G)\right)\right].
\end{equation*}

Using a test of equality of distribution, we calculate a test statistic; at a pre-specified level $\alpha$, if the test is well specified~\citep{shah2018hardness}, the one-sided \textit{p}-value is valid and corresponds to the probability under the null hypothesis $H_0: \PP^e(X_j|\PA_j^G) = \PP^{e'}(X_j|\PA_j^G)$ of a test statistic as large or larger than the observed test statistic. 

If a mechanism changes, a powerful test should yield a small \textit{p}-value and thus a term close to $1$ in the summation, similar to the ``hard" score. If a mechanism doesn't change, since \textit{p}-values are uniformly distributed in $[0, 1]$ under the null hypothesis, the term in the sum would be similarly uniformly distributed. With enough variables and environments, the variance of the sum of random uniform variables will decrease and the behavior of the score will be dominated by the \textit{p}-values under the alternatives. It must be noted that the \textit{p}-values are not independent, as some will use data from the same environments.

\clearpage
\section{Supporting experiments}\label{appendix:supporting_experiments}

\subsection{F1 Scores}

In the main text, we present the precision and recall separately as they are important metrics to consider. Here, we also present the F1 score which equals their harmonic mean and conveniently provides a single numeric summary. As before, \cref{fig:sim_empirical_rates_mss_f1} provides a simulated comparison of the \MSS\ estimator using various equality of distribution tests, while \cref{fig:sim_empirical_rates_contrast_f1} provides a simulated comparison of \MSS\ to other approaches in the literature which we discuss heavily in the main body.

\begin{figure}[t]
    \centering
    \includegraphics[width=\textwidth]{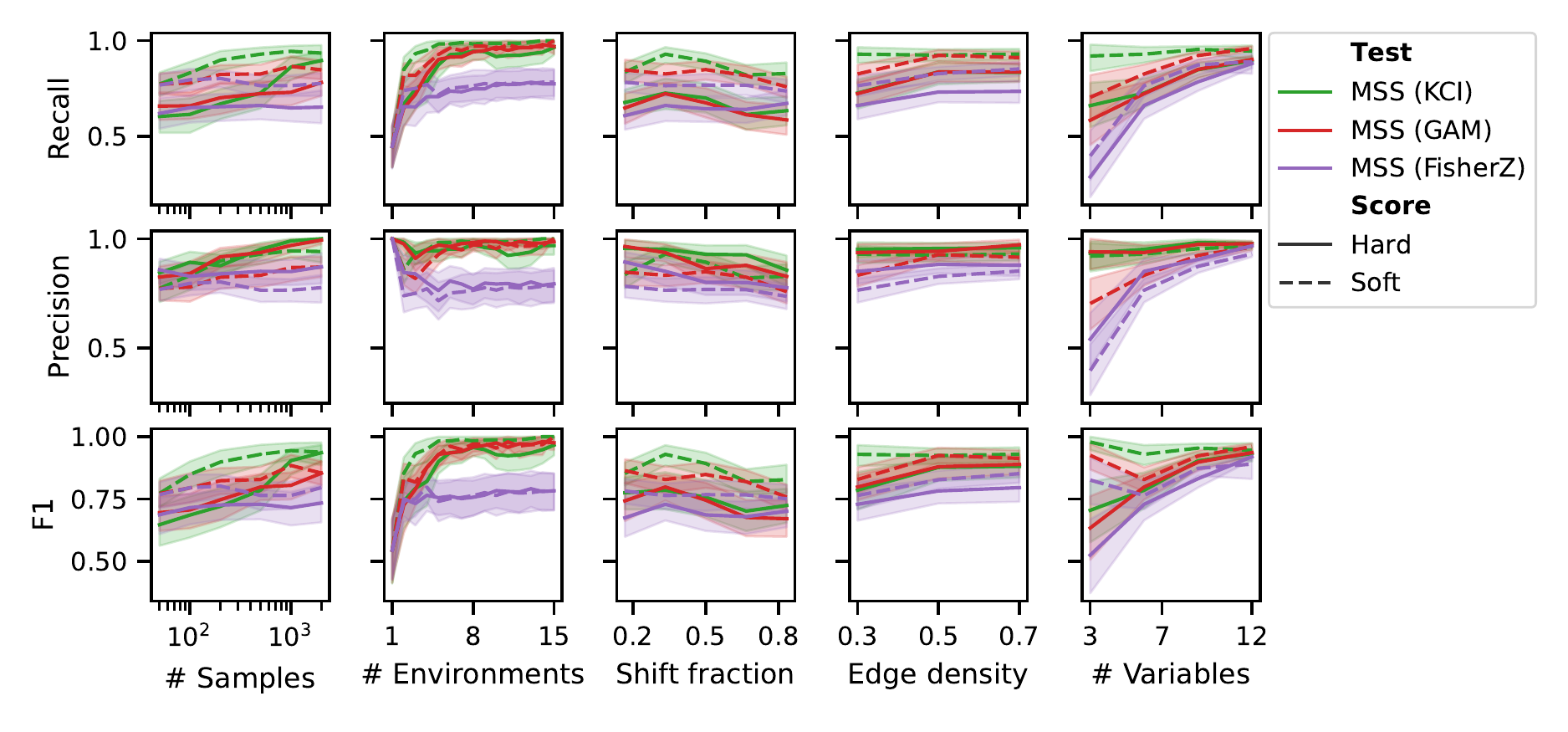}
    \caption{\small Nonparametric hypothesis tests perform well in nonlinear simulations, and soft scores succeed. Notably, recall converges with increasing environments. KCI appears to best balance high recall and precision.}
    \label{fig:sim_empirical_rates_mss_f1}
\end{figure}

\begin{figure}[t]
    \centering
    \includegraphics[width=\textwidth]{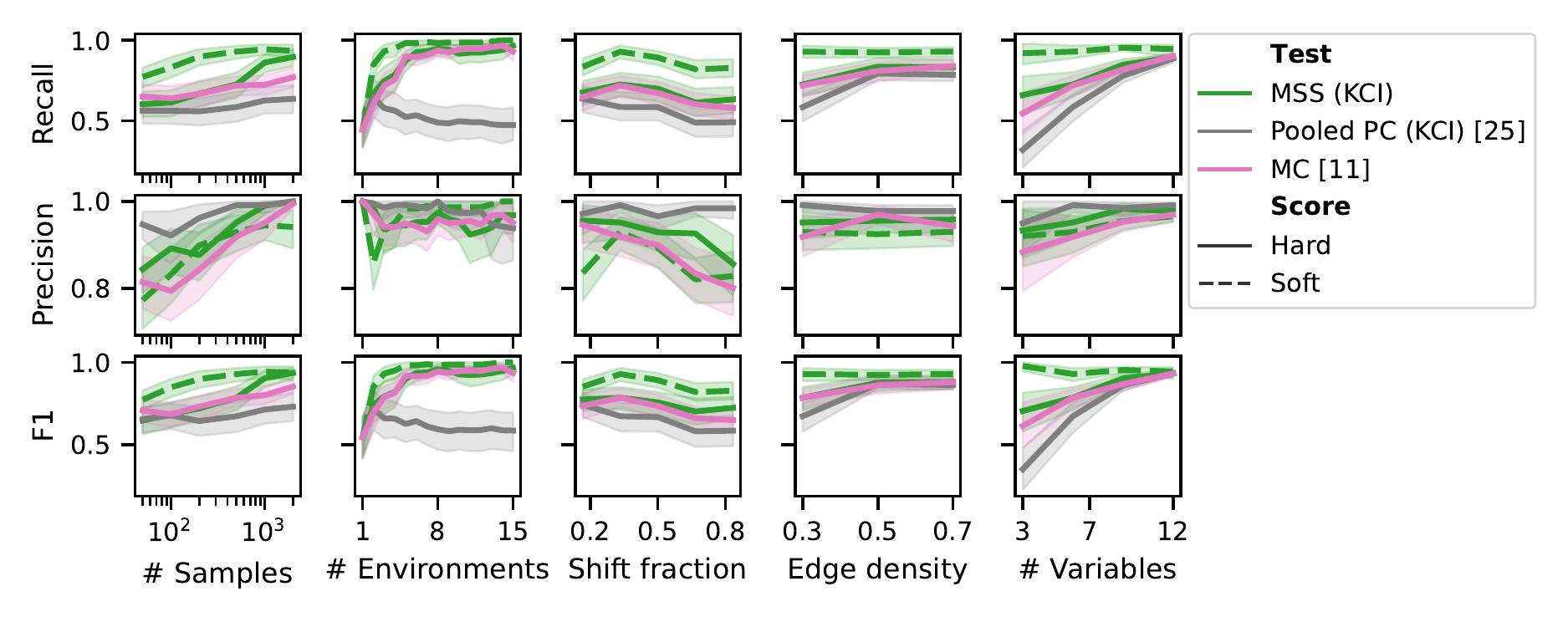}
    \caption{\small Pairwise approaches improve with more environments, unlike Pooled PC. Although the parametric MC works surprisingly well across settings, the nonparametric \MSS\ with KCI has superior precision and recall.}
    \label{fig:sim_empirical_rates_contrast_f1}
\end{figure}

\subsection{Additional simulations}

\subsubsection{\MSS\ improves upon pooled PC across random graph models.} Previously in~\cref{fig:sim_oracle_rates}, we demonstrated that the \MSS\ improves upon pooled PC under an oracle test in simulation settings where DAGs were sampled according to the Erd\H{o}s-R\`enyi (ER) random DAG model; in the ER model, each edge is sampled i.i.d. with a fixed probability. Here, we expand upon that simulation by further comparing rates under the Barabasi-Alberts (Hub) scale-free random DAG model; in the Hub model, vertices are sequentially added to the DAG and edges are connected to previous vertices with probability proportional to their existing number of edges.

As seen in~\cref{fig:sim_oracle_rates_dag_models}, we vary the same parameters as before but compare the two oracle methods across both DAG models this time. First, note that at $1$ environment the Hub model exhibits greater recall, indicating that the observational MEC of the Hub graph has fewer unoriented edges than that of the ER graph. Thus, the gap between the methods is lessened in a Hub model as compared to an ER model. Otherwise, the qualitative trends between the two methods are almost identical across the two random graph models. The \MSS\ appears at least mildly robust to the graph structure.

\begin{figure}[t]
    \centering
    \includegraphics[width=\textwidth]{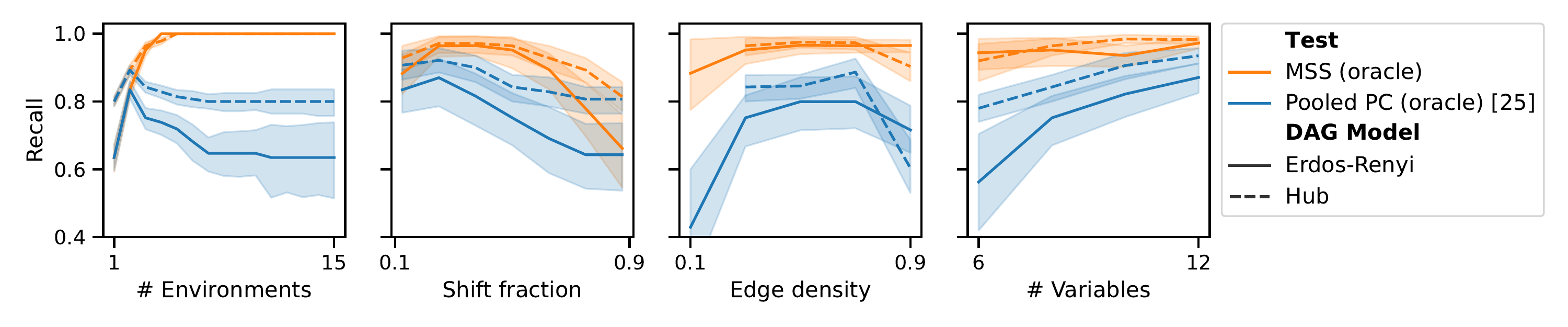}
    \caption{Oracle \MSS\ improves upon pooled PC across both random and hub random graph models.}
    \label{fig:sim_oracle_rates_dag_models}
\end{figure}

\subsubsection{Differences between oracle \MSS\ and pooled PC are most pronounced on sparser and smaller DAGs.} Although~\cref{fig:sim_oracle_rates} highlighted the most important trends of oracle methods in certain fixed settings, for completeness we examine rates of recall across additional fixed settings. As before, we sample DAGs from an Erd\H{o}s-R\`enyi distribution and in five environments vary the DAG density, shift fraction, and number of variables. The set of experimental results shown in~\cref{fig:sim_oracle_rates_relplot} convey broader trends in oracle recall rates as multiple variables change across row, column, and the x-axis. We do not vary the number of environments as we can only visualize three variables through our plot and the trend across environments is best understood from the theory. Differences in oracle recall rates are less pronounced on graphs with more variables and when the density of edges is large. Note that we only compare five environments here and that with more environments, differences will again be more pronounced; with enough environments, pooled PC cannot learn more than the MEC.

\begin{figure}[t]
    \centering
    \includegraphics[width=\textwidth]{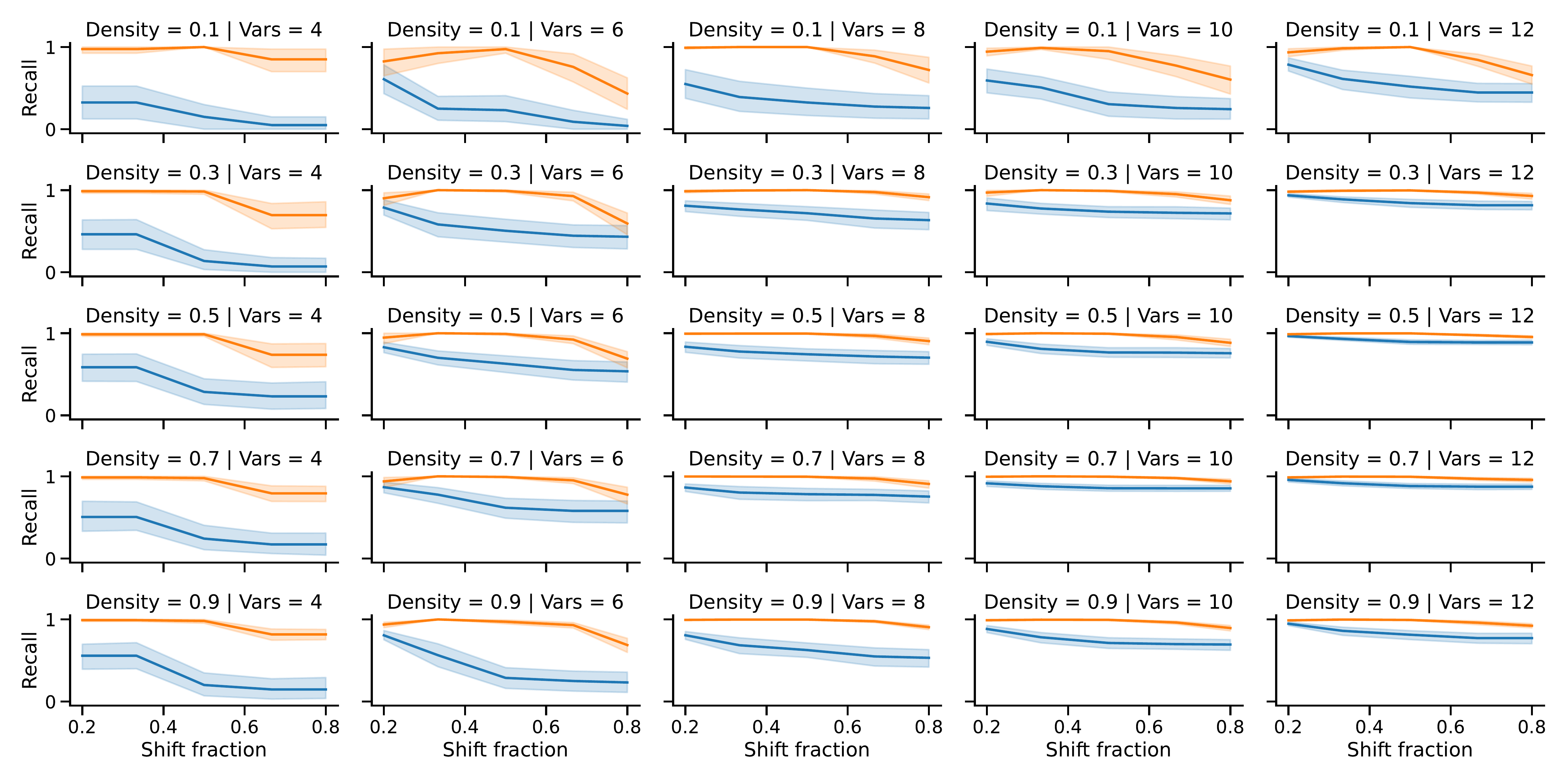}
    \caption{Differences between oracle \MSS\ and pooled PC on 5 environments are most pronounced on smaller and sparse DAGs. For readability, the legend is omitted but we refer back to the same legend in~\cref{fig:sim_oracle_rates}. Specifically, the orange line corresponds to the \MSS\ while the blue line corresponds to pooled PC. Only five environments are sampled, but differences would be exacerbated with additional environments.}
    \label{fig:sim_oracle_rates_relplot}
\end{figure}

\subsubsection{KCI-based approaches perform the best on bivariate CGMs.} We previously examined the empirical rates of recall and precision across various simulated settings, highlighting when methods succeed and fail. Due to the size and complexity of those studied DAGs, not all results are fully interpretable. We seek to further understand empirical performance through the simple bivariate DAG, which contains no indirect effects and few possible interventions to analyze. Specifically, on the DAG $X_1 \to X_2$, shifts can occur to either $\PP(X_1)$, $\PP(X_2 | X_1)$, neither mechanism, or both mechanisms; the first two shifts are sparse and provide oracle identifiability of the true DAG. Following the simulation setup described by~\cref{eq:sim_dist} on the DAG $X_1 \to X_2$, we simulate data from one base environment and from one interventional environment subject to one of the four possible shifts. Each environment has $500$ samples. We compare the four different \MSS\ methods using parametric and nonlinear equality of distribution tests and conditional independence tests. We also compare the pooled PC and MC approaches. Since only two environments are compared, we conjecture \MSS\ and pooled PC to be equivalent under an oracle test.

\begin{figure}[!htb]
    \centering
    \includegraphics[width=\textwidth]{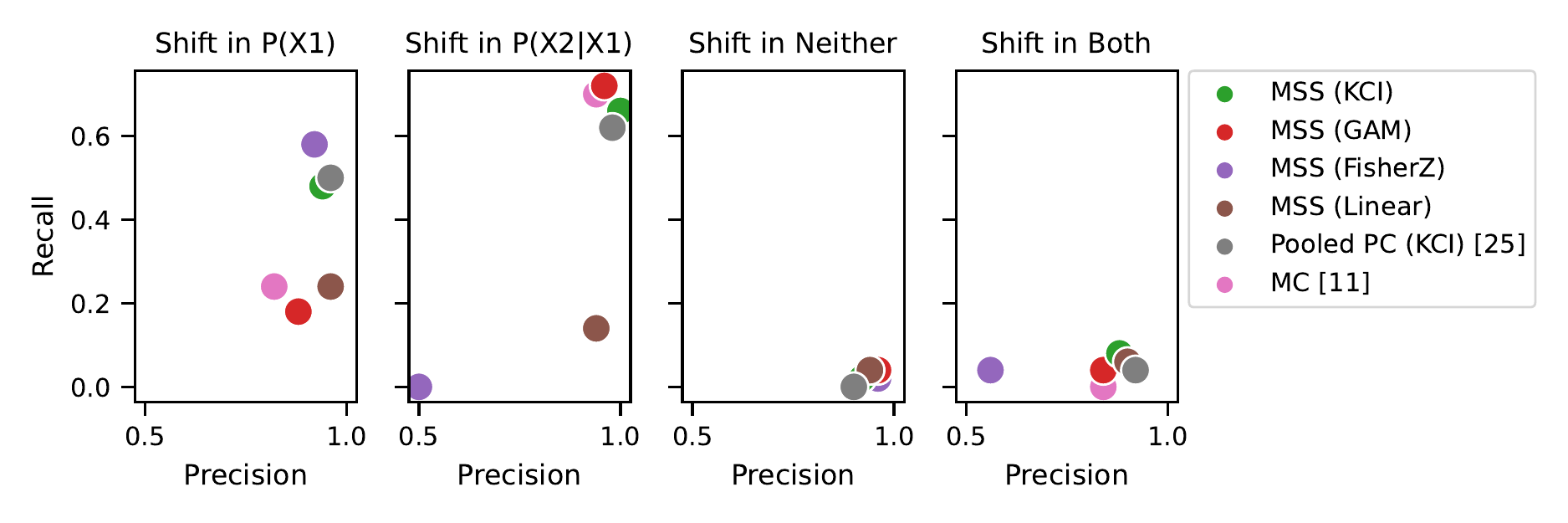}
    \caption{\small KCI-based tests perform well for causal identification in a bivariate CGM. $500$ samples are drawn from a base environment, and a second environment subject to one of four shifts given by the columns; the first two columns are sparse shift settings where we have identifiability. The precision and recall are plotted for each of the methods. It appears that the two KCI-based methods (\MSS\ and pooled PC) achieve the best balance of high power in both sparse shift settings while maintaining high precision in both non-sparse settings. Other methods either have drastically lower recall or precision close to $0.5$, indicating random guessing.}
    \label{fig:sim_bivariate_power}
\end{figure}

\begin{figure}[!htb]
    \centering
    \includegraphics[width=\textwidth]{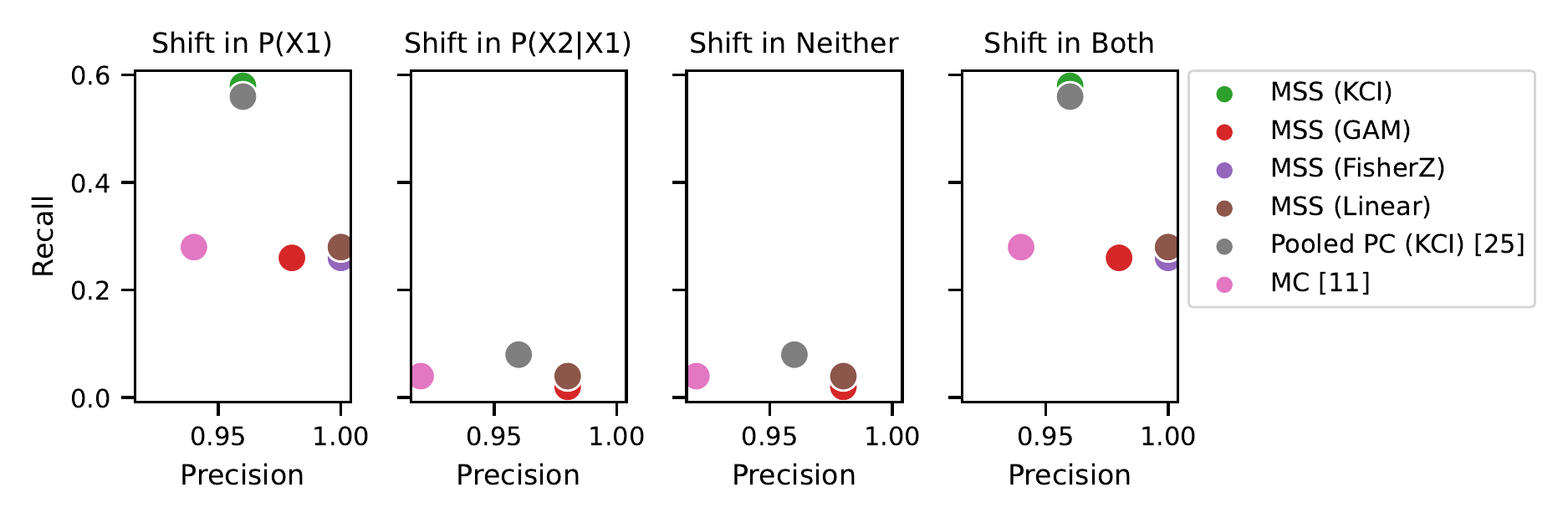}
    \caption{\small \looseness-1 \rebuttal{Test performance is dependent on the data generating process. In the same experimental setup as in} \cref{fig:sim_bivariate_power}\rebuttal{, we modify the simulation such that the noise is multiplicative. As we see, although the marginal change in $\PP(X_1)$ can still be detected, the conditional distribution tests are not powerful enough for shifts in $\PP(X_2 | X_1)$. However, notice that this means that the true graph is still identifiable when both mechanisms shift. This is contrary to theory about the oracle, \textit{exactly because the finite sample tests are not powerful enough.}}}
    \label{fig:sim_bivariate_multiplic_power}
\end{figure}

\looseness-1 Results are shown in~\cref{fig:sim_bivariate_power}. For reference, an oracle method would have recall $1$ in the first two (sparse shift) columns and $0$ in the other two columns.
Although Fisher-Z has high precision when $\PP(X_1)$ shifts, it has chance precision when $\PP(X_2|X_1)$ shifts.
The KCI methods maintain high precision while the precision of other methods is comparable or noticeably lower. With respect to recall, when neither or both mechanisms change and thus the DAG is not identifiable, all methods correctly have low recall. However, when just the marginal $\PP(X_1)$ changes, the KCI methods dominate in recall whereas the linear \MSS\, MS, and GAM approaches have lower recall, implying they are less often able to detect a change in the reverse conditional $\PP(X_1~|~X_2)$. When the mechanism $\PP(X_2~|~X_1)$ shifts, all methods have high recall. Notably, the linear \MSS\ performs much worse than MC. The only difference between them, however, is that MC explicitly counts how many parameters change while the linear \MSS\ simply tests if there is a change; this does come at a slight cost in precision for MS though. \rebuttal{In a small extension, we additionally run this experiment under a multiplicative noise data generating process. Those results are seen in~\mbox{\cref{fig:sim_bivariate_multiplic_power}} and highlight both that it is necessary to have access to a powerful hypothesis test and yet failing to reject the null can promote sparsity and lead to identifiability under dense changes.}.

\subsection{Application to real-world cytometry data}
\label{appendix:cytometry}

Although simulations with known ground truth provide useful reference points for comparing methods and evaluating empirical performance, in practice we are interested in studying real data with no known truth and additional challenges such as violated assumptions. To illustrate how one may apply our method in practice, and to analyze empirical performance on real data, we conduct a case-study application of \MSS\, for causal discovery on a well-studied cytometry dataset~\citep{sachs2005causal}.

\subsubsection{Background}

\citet{sachs2005causal} present a detailed study of the application of Bayesian discovery approaches to learning a causal DAG among protein concentration levels in human immune system cells. In each of $9$ experimental environments subject to different perturbations, approximately $700$-$900$ sample measurements were collected; each sample is the concentration levels of $11$ proteins from a cell. The learned \textit{Sachs network} is a proposed DAG among the variables, which the authors discuss and contrast with a domain-expert network from the ``biologist's view''. This cytometry data has subsequently been studied in further detail~\citep{eaton2007exact, ramsey2018fask, mooij2020joint}. As is often pointed out, various assumptions may be violated, including the acyclicity assumption, since protein networks contain strong feedback loops~\citep{mooij2020joint}. As such, it is not necessarily useful to treat the Sachs network as a ground truth and there are numerous relationships and orientations which should rightfully be questioned~\citep{mooij2020joint}. Results must be considered in the context of domain-knowledge and various existing studies in the literature.

\subsubsection{Experimental setup}

In order to focus on learning edge orientations of undirected edges in the MEC, rather than learning the MEC, we start from the Sachs network despite the potential caveats. The Sachs network from~\citet{sachs2005causal} is a DAG on the $11$ variables with $17$ edges. We compute the \textit{Sachs MEC} which contains all DAGs which are Markov equivalent to the Sachs network. The Sachs MEC has no directed edges, and thus is simply the undirected skeleton of the graph in~\cref{fig:unified_empirical}. Starting from the Sachs MEC makes our results more interpretable in light of previous works and saves costly computation of the MEC. In practice, we would advise starting from the pooled PC MEC; based on the number of environments and observed density of changes, we would not expect this to orient any edges beyond the observational MEC.

Starting from the Sachs MEC, we apply the \MSS\ using the KCI test, which appears to perform the best among plug-in estimators for \MSS\ in our simulations. Since the feature distributions are heavily skewed, we preprocess them by taking their natural logarithm~\citep{ramsey2018fask}. Among all DAGs in the Sachs MEC, the DAG with the uniquely minimal number of shifts exhibits approximately $8.9$ shifts per pair of environments; this is relatively high but satisfies the assumption of sparse shifts. Violations to assumptions may lead to more shifts than expected.

\subsubsection{Results and comparison to related works}

The DAG which minimizes the \MSS\ is the unique minimizer and is visualized in~\cref{fig:unified_empirical}. An edge in black is oriented in the same direction as in the Sachs network, while an edge in red is oriented in the opposite direction. Overall, the majority of edges match the Sachs network. The edges which do not match, however, reflect ambiguities and flawed assumptions. We list each edge which does not match the Sachs network and discuss why this might be the case in light of existing work.
\begin{itemize}
    \item PIP2 $\to$ PIP3: As illustrated in \citet{sachs2005causal}, these two proteins are actually cyclically related through bi-directed edges in the accepted ``biologist's view''. Indeed, PIP2 $\to$ PIP3 was similarly recovered by an analysis of \citet{ramsey2018fask}, detailed in their Figure 11.
    \item Mek $\to$ Raf: that this edge does not match the Sachs network is discussed heavily by~\citet{mooij2020joint} who point to it as a fundamental flaw of the Sachs network. The Mek $\to$ Raf edge is indeed found by many other methods~\citep{eaton2007exact, ramsey2018fask, mooij2020joint}.
    \item The PKA, PKC, P38 triangle: although there is not a detailed discussion of these variables in other studies, there is strong ambiguity in the edge directions among approaches. Notably, \citet{mooij2020joint} similarly find strong evidence in their approach for the edge P38 $\to$ PKC while \citet{ramsey2018fask} and \citet{eaton2007exact} find evidence for PKA $\to$ PKC. However, all other approaches agree that the edge P38 $\to$ PKA is incorrect. Although we do not explore further, it is worth noting that the 3rd minimal \MSS\ DAG (not shown) is the same as the one shown, except it contains the presumed correct edge PKA $\to$ P38.
\end{itemize}

As an additional note, we see in~\cref{fig:unified_empirical} that the edge Mek $\to$ Erk is correctly recovered. \citet{sachs2005causal} similarly recover this well-known connection and point to it as strong evidence of success. In contrast, neither \citet{eaton2007exact}, \citet{ramsey2018fask}, nor \citet{mooij2020joint} recover the edge with their methods. Indeed, \citet{ramsey2018fask} specifically discuss how their approach incorrectly missed this edge, potentially the result of signal being lost when all the data is pooled. Pooled PC would face a similar issue, exacerbated by additional environments, while the pairwise comparisons of the \MSS\ help to avoid this issue.

\end{document}